\documentclass[11pt]{article}
\usepackage{amsfonts, amsmath, fullpage, rotating, amssymb, enumitem,amssymb}
\usepackage{color,subfig}
\usepackage{float} 
\usepackage{multirow}
\usepackage{authblk}
\usepackage{algorithm,algpseudocode}
\usepackage[ntheorem]{mdframed}
\usepackage{subfiles}
\usepackage{blindtext}
\usepackage{tcolorbox}
\tcbuselibrary{xparse}

\usepackage{mathtools,thmtools,amsthm}

\newtheorem{theorem}{Theorem}
\newtheorem{lemma}{Lemma}[section]

\newtheorem{claim}{Claim}[lemma]

\theoremstyle{remark}
\newtheorem{remark}{Remark}

\theoremstyle{definition}
\newtheorem{definition}{Definition}
\newtheorem{problem}{Problem}
\newtheorem{exmp}[lemma]{Example}

\DeclarePairedDelimiter\floor{\lfloor}{\rfloor}

\DeclareTColorBox{conclusion}{}{fonttitle=\bfseries\sffamily\large, title=Conclusion, colframe=red!75!black,colback=red!5!white}

\def\FullBox{\hbox{\vrule width 8pt height 8pt depth 0pt}}

\def\qed{\ifmmode\qquad\FullBox\else{\unskip\nobreak\hfil
\penalty50\hskip1em\null\nobreak\hfill$\Box$
\parfillskip=0pt\finalhyphendemerits=0\endgraf}\fi}

\def\qedsketch{\ifmmode\Box\else{\unskip\nobreak\hfil
\penalty50\hskip1em\null\nobreak\hfil$\Box$
\parfillskip=0pt\finalhyphendemerits=0\endgraf}\fi}

\newcommand{\by}{\mathbf{y}}
\newcommand{\bx}{\mathbf{x}}
\newcommand{\bb}{\mathbf{b}}

\newcommand{\bz}{\mathbf{z}}

\newcommand{\bv}{\mathbf{v}}

\newcommand{\bI}{\mathbf{I}}
\newcommand{\bu}{\mathbf{u}}

\newcommand{\bs}{\mathbf{s}}

\newcommand{\ATA}{A^\top A}
\newcommand{\ATb}{A^\top\bb}

\newcommand{\Dt}{\Delta t}

\newcommand{\bmu}{\boldsymbol{\mu}}

\newcommand{\OPT}{\mathbf{OPT}}

\DeclareMathOperator*{\argmin}{arg\,min}

\newcommand{\ie} {{\it i.e.,\ }}

\newcommand{\R}{{\mathbb R}} 
\newcommand{\N}{{\mathbb{N}}} 
\newcommand{\poly}{{\mathrm{poly}}}
\newcommand{\polylog}{{\mathrm{polylog}}}

\newcommand{\class}[1]{\mathbf{#1}}

\newcommand{\NP}{\class{NP}}

\usepackage{graphicx}
\usepackage{tikz}
\usetikzlibrary{positioning}
\usetikzlibrary{shapes,arrows}
\usetikzlibrary{shadows,arrows}
\graphicspath{{images/}{../images/}}
\allowdisplaybreaks

\usepackage{xspace}
\RequirePackage[colorlinks=true]{hyperref}
\hypersetup{
  linkcolor=[rgb]{0,0,0.5},
  citecolor=[rgb]{0, 0.5, 0},
  urlcolor=[rgb]{0.5, 0, 0}
}

\title{On the Algorithmic Power of Spiking Neural Networks}

\author{Chi-Ning Chou\thanks{School of Engineering and Applied Sciences, Harvard University, USA. Supported by NSF awards CCF 1565264 and CNS 1618026. Email: \texttt{chiningchou@g.harvard.edu}.}}
\author{Kai-Min Chung\thanks{Institute of Information Science, Academia Sinica, Taipei, Taiwan.}}
\author{Chi-Jen Lu\thanks{Institute of Information Science, Academia Sinica, Taipei, Taiwan.}}
\affil[]{}

\date{\today} 

\begin{document}
	
\maketitle

\begin{abstract}
\emph{Spiking Neural Networks} (SNN) are mathematical models in neuroscience to describe the dynamics among a set of neurons that interact with each other by firing instantaneous signals, \emph{a.k.a.}, \emph{spikes}. 
Interestingly, a recent advance in neuroscience [Barrett-Den\`{e}ve-Machens, NIPS 2013] showed that the neurons' \emph{firing rate}, \emph{i.e.,} the average number of spikes fired per unit of time, can be characterized by the optimal solution of a quadratic program defined by the parameters of the dynamics. This indicated that SNN potentially has the computational power to solve non-trivial quadratic programs. However, the results were justified empirically without rigorous analysis. 

We put this into the context of \emph{natural algorithms} and aim to investigate the algorithmic power of SNN. Especially, we emphasize on giving rigorous asymptotic analysis on the performance of SNN in solving optimization problems. To enforce a theoretical study, we first identify a simplified SNN model that is tractable for analysis. Next, we confirm the empirical observation in the work of Barrett et al. by giving an upper bound on the convergence rate of SNN in solving the quadratic program. Further, we observe that in the case where there are infinitely many optimal solutions, SNN tends to converge to the one with smaller $\ell_1$ norm. We give an affirmative answer to our finding by showing that SNN can solve the $\ell_1$ minimization problem under some regular conditions.

Our main technical insight is a \emph{dual view} of the SNN dynamics, under which SNN can be viewed as a new natural primal-dual algorithm for the $\ell_1$ minimization problem. We believe that the dual view is of independent interest and may potentially find interesting interpretation in neuroscience. 
\end{abstract}

\vspace{10pt} 
	
\section{Introduction}\label{sec:intro}
The theory of \textit{natural algorithms} is a framework that bridges the algorithmic thinking in computer science and the mathematical models in biology. 
Under this framework, biological systems are viewed as \textit{algorithms} to \emph{efficiently} solve specific \textit{computational problems}.
Seminal works such as bird flocking~\cite{chazelle2009natural,chazelle2012natural}, slime systems~\cite{nakagaki2000intelligence, tero2007mathematical, bonifaci2012physarum}, and evolution \cite{livnat2014satisfiability,livnat2016sex} successfully provide algorithmic explanations for different natural objects. These works give rigorous theoretical results to confirm empirical observations, shed new light on the biological systems through computational lens, and sometimes lead to new biologically inspired algorithms.

In this work, we investigate \textit{Spiking Neural Networks (SNNs)} as natural algorithms for solving convex optimization problems. SNNs are mathematical models for biological neural networks where a network of neurons transmit information by \textit{firing spikes} through their synaptic connections (\textit{i.e.,} edges between two neurons).
Our starting point is a seminal work of Barrett, Den\`{e}ve, and Machens \cite{barrett2013firing}, where they showed that the \emph{firing rate} (\textit{i.e.,} the average number of spikes fired by each neuron) of a certain class of \textit{integrate-and-fire} SNNs can be characterized by the optimal solutions of a quadratic program defined by the parameters of SNN. Thus, the SNN can be viewed as a natural algorithm for the corresponding quadratic program. However, no rigorous analysis was given in their work. 

We bridge the gap by showing that the firing rate converges to an optimal solution of the corresponding quadratic program with an explicit polynomial bound on the convergent rate. Thus, the SNN indeed gives an \textit{efficient algorithm} for solving the quadratic program. To the best of our knowledge, this is the first result with an explicit bound on the convergent rate. Previous works~\cite{shapero2013configurable,shapero2014optimal,tang2017sparse} on related SNN models for optimization problems are either heuristic or only proving convergence results when the time goes to infinity (see Section~\ref{sec:intro related} for full discussion on related works).

We take one step further to ask what other optimization problems can SNNs efficiently solve. As our main result, we show that when configured properly, SNNs can solve the \textit{$\ell_1$ minimization problem}\footnote{The problem is defined as given matrix $A\in\R^{m\times n}$, vector $\bb\in\R^m$, and guaranteed that there is a solution to $A\bx=\bb$. The goal is finding a solution $\bx$ with the smallest $\ell_1$ norm. See Section~\ref{sec:preliminaries} for formal definition.} in polynomial time\footnote{The running time is polynomial in a parameter depending on the inputs. In some cases, this parameter might cause the running to be quasi-polynomial or sub-exponential. See Section~\ref{sec:nice} for more details.}. Our main technical insight is interpreting the dynamics of SNNs in a \textit{dual space}. In this way, SNNs can be viewed as a new primal-dual algorithm for solving the \textit{$\ell_1$ minimization problem}. 

In the rest of the introduction, we will first briefly introduce the background of spiking neural networks (SNNs) and formally define the mathematical model we are working on. Next, our results will be presented and compared with other related works. Finally, we wrap up this section with potential future research directions and perspectives.

\subsection{Spiking Neural Networks}\label{sec:intro SNN model}

Spiking neural networks (SNNs) are mathematical models for the dynamics of biological neural networks.
An SNN consists of neurons, and each of them is associated with an intrinsic electrical charge called \textit{membrane potential}. When the potential of a neuron reaches a certain level, it will fire an instantaneous signal, \textit{i.e.,} \textit{spike}, to other neurons and increase or decrease their potentials.

Mathematically, the dynamic of neuron's membrane potential in an SNN is typically described by a \textit{differential equation}, and there are many well-studied models such as the \textit{integrate-and-fire model}~\cite{lapicque1907recherches}, the \textit{Hodgkin-Huxley model}~\cite{hodgkin1952quantitative}, and their variants~\cite{fitzhugh1961impulses,stein1965theoretical,morris1981voltage,hindmarsh1984model,gerstner1995time,kistler1997reduction,brunel2003firing,fourcaud2003spike,izhikevich2003simple,teka2014neuronal}. 
In this work, we focus on the integrate-and-fire model defined as follows.
Let $n$ be the number of neurons and $\bu(t)\in\R^n$ be the vector of membrane potentials where $\bu_i(t)$ is the potential of neuron $i$ at time $t$ for any $i\in[n]$ and $t\geq0$. The dynamics of $\bu(t)$ can be described by the following differential equation: for each $i\in[n]$ and $t\geq0$
\begin{equation}\label{eq:SNN IAF differential equation}
\frac{d}{dt}\bu_i(t) = \sum_{j\in[n]}-C_{ji}(t)\bs_j(t) + \bI_i(t)
\end{equation}
where the initial value of the potentials are set to 0, \ie $\bu_i(0)=0$ for each $i\in[n]$.
There are two terms that determine the dynamics of membrane potentials as shown in~\eqref{eq:SNN IAF differential equation}. The simpler term is the input charging\footnote{Also known as \textit{input signal} or \textit{input current}.} $\bI(t)\in\R^n$, which can be thought of as an external effect on each neuron. The other term models the instantaneous spike effect among neurons. Specifically, the $-C_{ji}(t)\bs_j(t)$ term models the effect on the potential of neuron $i$ when neuron $j$ fires a spike. Here $C(t)\in\R^{n\times n}$ is the connectivity matrix that encodes the \textit{synapses} between neurons, where $C_{ji}(t)$ describes the connection strength from neuron $j$ to neuron $i$. $\bs(t)\in\R^n$ is the \textit{spike train} that records the spikes of each neuron, and $\bs_i(t)$ can be thought of as indicating whether neuron $i$ fires a spike at time $t$. To sum up, the $-C_{ji}(t)\bs(t)$ term decreases\footnote{If $C_{ji}(t^*)<0$, then the potential of neuron $i$ actually \textit{increases by $|C_{ji}(t^*)|$.}} the potential of neuron $i$ by $C_{ji}(t^*)$ whenever neuron $j$ fires a spike at time $t^*$.

The spike train $\bs(t)$ is determined by the spike events, which are in turn determined by the spiking rule. A typical spiking rule is the threshold rule. Specifically, let $\eta>0$ be the spiking threshold, the threshold rule simply says that neuron $i$ fires a spike at time $t$ if and only if $\bu_i(t)>\eta$. Next, record the timings when neuron $i$ fires a spike as $0\leq t_1^{(i)}<t_2^{(i)}<\dots$ and let $k_i(t)$ be the number of spikes within time $[0,t]$. An important statistics of the dynamics is the \textit{firing rate} defined as $\bx_i(t) := k_i(t) / t$ for neuron $i\in[n]$ at time $t$, namely, the average number of spikes of neuron $i$ up to time $t$. The last thing we need for specifying $\bs(t)$ is the \textit{spike shape}, which can be modeled as a function $\delta:\R_{\geq0}\rightarrow\R$. Intuitively, the spike shape describes the effect of a spike, and standard choices of $\delta$ could be the Dirac delta function or a pulse function with an exponential tail. Now we can define $\bs_i(t)=\sum_{1\leq s\leq k_i(t)}\delta(t-t_s^{(i)})$ to be the \textit{spike train} of neuron $i$ at time $t$.


We provide the following concrete example to illustrate the SNN dynamics introduced above.

\begin{exmp}
	Let $n=2$, $\eta=1$, and $\delta$ be the Dirac delta function such that for any $\epsilon>0$, $\int_0^\epsilon\delta(t)dt=1$ and $\delta(t)\geq0$ for any $t\geq0$. Let both input charging and connectivity matrix be \textit{static}, \ie $\bI(t)=\bI$ and $C(t)=C$ for any $t\geq0$, and consider
	\[
	C = \begin{pmatrix}1&0\\-0.1&1
	\end{pmatrix},\ \bI=\begin{pmatrix}0.1\\0
	\end{pmatrix},\text{ and }\bu(0)=\begin{pmatrix}0\\0\end{pmatrix}.
	\]
	
	In Figure~\ref{fig:example}, we simulate this SNN for 500 seconds. 
	We can see that neuron 1 fires a spike every ten seconds while neuron 2 fires a spike every one hundred seconds. As a result, the firing rate of neuron 1 will gradually converge to 0.1 and that of neuron 2 will go to 0.01.
	
	\begin{figure}[h]
		\centering
		\includegraphics[width=14cm]{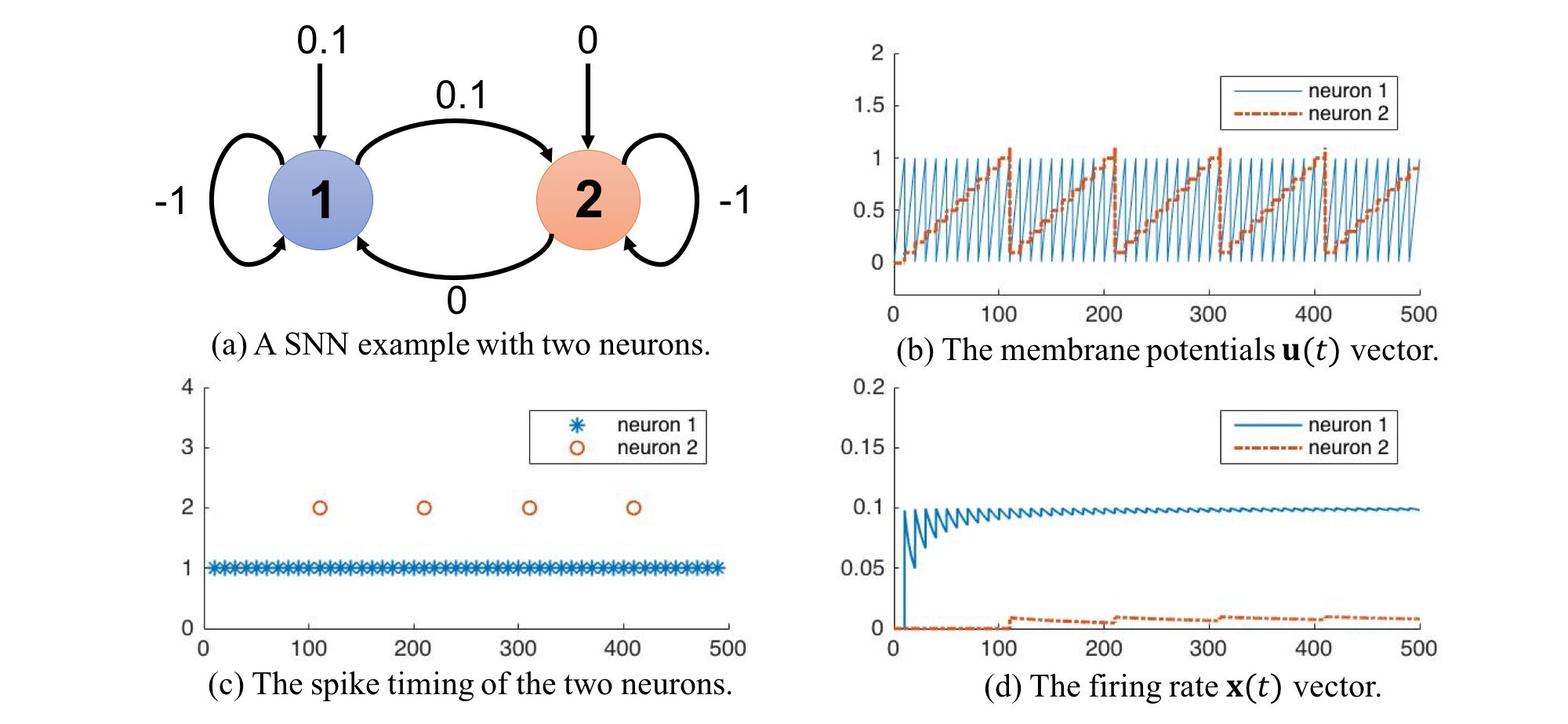}
		\caption{The example of SNN with two neurons. In (a), we describe the dynamic of this SNN. Note that the effect of spikes is the negation of the synapse encoded in the connectivity matrix $C$. In (b), we plot the membrane potential vectors $\bu(t)$. In (c), we plot the timings when neurons fire a spike. One can see that neuron 1 fires a spike every ten seconds while neuron 2 fires a spike every one hundred seconds. In (d), we plot the firing rate vector $\bx(t)$. One can see that the firing rate of neuron 1 will gradually converge to 0.1 and that of neuron 2 will go to 0.01.}
		\label{fig:example}
	\end{figure}
\end{exmp}

In general, both the input charging vector $\bI(t)$ and the connectivity matrix $C(t)$ can evolve over time, in which the change of $\bI(t)$ models the variation of the environment and the change of $C_{ji}(t)$ captures the adaptive \textit{learning} behavior of the neurons to the environmental change. Understanding how synapses evolve over time (\ie synapse plasticity) is a very important subject in neuroscience.
However, in this work, we follow the choice of Barrett et al.~\cite{barrett2013firing} and consider \emph{static} SNN dynamics, where both the input charging $\bI(t)$ and the synapses $C(t)$ are constants. Although this is a special case compared to the general model in~\eqref{eq:SNN IAF differential equation}, we justify the choice of static SNN by showing that SNN already exhibits non-trivial computational power even in this restricted model.

As in Barrett et al.~\cite{barrett2013firing}, we focus on static SNN and view it as a natural algorithm for optimization problems. Specifically, given an instance to the optimization problem, the goal is to configure a static SNN (by setting its parameters) so that the firing rate converge to an optimal solution efficiently. In this sense, the result of Barrett et al.~\cite{barrett2013firing} can be interpreted as a natural algorithm for certain quadratic programs. In our eyes, the solution being encoded as the firing rate is an interesting and peculiar feature of the SNN dynamics. 
Also, the dynamics of a static SNN can be viewed as a simple distributed algorithm with a simple communication pattern. Specifically, once the dynamics is set up, each neuron only needs to keep track of its potential and communicate with each other through spikes.

\subsection{Our Results}

Barrett et al.~\cite{barrett2013firing} gave a clean characterization of the firing rates by the network connectivity and input signal. Concretely, they considered \textit{static} SNN where both the connectivity matrix $C\in\R^{n\times n}$ and the external charging $\bI\in\R^n$ do not change with time. They argued that the firing rate would converge to the solution of the following quadratic program.
\begin{equation}\label{op:quadratic program}
\begin{aligned}
& \underset{\bx\in\R^n}{\text{minimize}}
& & \|C\bx-\bI\|_2^2 \\
& \text{subject to}
& & \bx_i\geq0,\ \forall i\in[n].
\end{aligned}
\end{equation}
They supported this observation by giving simulations on the so called \textit{tightly balanced networks} and yielded pretty accurate predictions in practice. 
Also, they heuristically explained the reason how they came up with the quadratic program. 
However, no rigorous theorem had been proved on the convergence of firing rate to the solution of this quadratic program.\\

To give a theoretical explanation for the discovery of~\cite{barrett2013firing}, we start with a simpler SNN model to enable the analysis.

\paragraph{The simple SNN model}
In the simple SNN model, we make two simplifications on the general model in~\eqref{eq:SNN IAF differential equation}. 

First, we pick the shape of spike to be the Dirac delta function. That is, let $\delta(t) = \mathbf{1}_{t=0}$ and thus $\bs_i(t)=\mathbf{1}_{\bu_i(t)>\eta}$. This simplification saves us from complicated calculation while the Dirac delta function still captures the instantaneous behavior of a spike.

Second, we consider the connectivity matrix $C$ in the form $C=\alpha\cdot A^\top A$ where $\alpha>0$ is the spiking strength and $A\in\R^{m\times n}$ is the Cholesky decomposition of $C$. The reason for introducing $\alpha$ is to model the height of the Dirac delta function. Mathematically, it is redundant to have both $\alpha$ and $C$ since the model remains the same when combining $\alpha$ with $C$. However, as we will see in the next subsection, separating $\alpha$ and $C$ is meaningful as $C$ corresponds to the \textit{input} of the computational problem and $\alpha$ is the parameter that one can choose to configure an SNN to solve the problem.

In this work, we focus on the algorithmic power of SNN in the following sense. Given a problem instance, one configures a SNN and sets the firing rate $\bx(t)$ to be the output at time $t$. We say this SNN solves the problem if $\bx(t)$ converges to the solution of the problem.

\paragraph{Simple SNN solves the non-negative least squares.} As mentioned, Barrett et al.~\cite{barrett2013firing} identified a connection between the firing rate of SNN with integrate-and-fire neurons and a quadratic programming problem~\eqref{op:quadratic program}. They gave empirical evidence for the correctness of this connection, however, no theoretical guarantee had been provided. Our first result confirms their observation by giving the first theoretical analysis. Specifically, when $C=A^\top A$ and $\bI=A^\top\bb$, the firing rate will converge to the solution of the following \textit{non-negative least squares problem}.

\begin{equation}\label{op:NNLS}
\begin{aligned}
& \underset{\bx\in\R^n}{\text{minimize}}
& & \|A\bx-\bb\|_2^2 \\
& \text{subject to}
& & \bx_i\geq0,\ \forall i\in[n].
\end{aligned}
\end{equation}

\begin{theorem}[informal]\label{thm:quadratic program informal}
	Given $A\in\R^{m\times n}$, $\bb\in\R^m$, and $\epsilon>0$. Suppose $A$ satisfies some regular conditions\footnote{More details about the regular conditions will be discussed in Section~\ref{sec:nice}.}. Let $\bx(t)$ be the firing rate of the simple SNN with $0<\alpha\leq\alpha(A)$ where $\alpha(A)$ is a function depending on $A$. When $t\geq\Omega(\frac{\sqrt{n}}{\epsilon\cdot\|\bb\|_2})$,\footnote{The $\Omega(\cdot)$ and the $O(\cdot)$ later both hide the dependency on some parameters of $A$. See Section~\ref{sec:nice}.} $\bx(t)$ is an $\epsilon$-approximate solution\footnote{See Definition~\ref{def:non-negative least squares eps approx sol} for the formal definition of $\epsilon$-approximate solution.} for the non-negative least squares problem of $(A,\bb)$.
\end{theorem}

See Theorem~\ref{thm:quadratic program} in Section~\ref{sec:quadratic program} for the formal statement of this theorem. To the best of our knowledge, this is the first\footnote{See Section~\ref{sec:intro related} for comparisons with related works.} theoretical result on the analysis of SNN with an explicit bound on the convergence rate and shows that SNN can be implemented as an efficient algorithm for an optimization problem. 

\paragraph{Simple SNN solves the $\ell_1$ minimization problem.} 
In addition to solving the non-negative least squares problem, as our main result, we also show that the simple SNN is able to solve the \textit{$\ell_1$ minimization problem}, which is defined as minimizing the $\ell_1$ norm of the solutions of $A\bx=\bb$.
$\ell_1$ minimization problem is also known as the {\it basis pursuit} problem proposed by Chen et al. \cite{chen2001atomic}. The problem is widely used for recovering sparse solution in compressed sensing, signal processing, face recognition etc. 

Before the discussion on $\ell_1$ minimization, let us start with a digression on the \textit{two-sided} simple SNN for the convenience of future analysis.
$$
\frac{d}{dt}\bu(t) = -\alpha\cdot A^\top A\bs(t) +A^\top\bb
$$
where $\bs_i(t)=\mathbf{1}_{\bu_i(t)>\eta}-\mathbf{1}_{\bu_i(t)<-\eta}$.
Note that the two-sided SNN is a special case of the one-sided SNN in the sense that one can use the one-sided SNN to simulate the two-sided SNN as follows. Given a two-sided SNN described above with connectivity matrix $C=A^\top A$ and external charging $\bI=A^\top\bb$. Let $C'=\bigl( \begin{smallmatrix}A^\top A & -A^\top A \\ -A^\top A & A^\top A\end{smallmatrix}\bigr)$ and $\bI'=\bigl( \begin{smallmatrix} A^\top\bb \\ -A^\top\bb\end{smallmatrix}\bigr)$. Intuitively, this can be thought of as duplicating each neuron and flip its connectivities with other neurons.

To solve the $\ell_1$ minimization problem, we simply configure a two-sided SNN as follows. Given an input $(A,\bb)$, let $C=A^\top A$ and $\bI=A^\top\bb$. Now, we have the following theorem.

\begin{theorem}[informal]\label{thm: l1 informal}
	Given $A\in\R^{m\times n}$, $\bb\in\R^m$, and $\epsilon>0$. Suppose $A$ satisfies some regular conditions. Let $\bx(t)$ be the firing rate of the two-sided simple SNN with $0<\alpha\leq\alpha(A)$ where $\alpha(A)$ is a function depending on $A$. When $t\geq\Omega(\frac{n^3}{\epsilon^2})$, $\bx(t)$ is an $\epsilon$-approximate solution\footnote{See Definition~\ref{def:l1 min eps approx sol} for the formal definition of $\epsilon$-approximate solution.} for the $\ell_1$ minimization problem of $(A,\bb)$.
\end{theorem}

See Theorem~\ref{thm:l1} for the formal statement of this theorem.
As we will discuss in the next subsection, under the dual view of the SNN dynamics, the simple two sided SNN can be interpreted as a new natural primal-dual algorithm for the $\ell_1$ minimization problem.

\subsection{A Dual View of the SNN Dynamics} \label{subsec:dual-view}

The main techniques in this work is the discovery of a \textit{dual view} of SNN. Recall that the dynamics of a static SNN can be described by the following differential equation.
$$
\frac{d}{dt}\bu(t) = -\alpha\cdot C\bs(t) + \bI
$$
where $\bu(0)=\mathbf{0}$ the parameters $C$ and $\bI$ can be represented as $C=A^\top A$ and $\bI=A^\top\bb$ for some $A\in\R^{m\times n}$ and $\bb\in\R^m$. For simplicity, we pick the firing threshold $\eta=1$ here. Let us call the dynamics of $\bu(t)$ the \textit{primal SNN}. Now, the \textit{dual SNN}, can be defined as follows.
$$
\frac{d}{dt}\bv(t) = -\alpha\cdot A\bs(t) + \bb
$$
where $\bv(0)=\mathbf{0}$ and $\bs(t)$ defined as the usual way. At first glance, this merely looks like a simple linear transformation, Nevertheless, the dual SNN provides a nice \textit{geometric view} for the SNN dynamics as follows.

\begin{figure}[h]
	\centering
	\subfloat[An example of one neuron.]{%
		\includegraphics[width=0.45\textwidth]{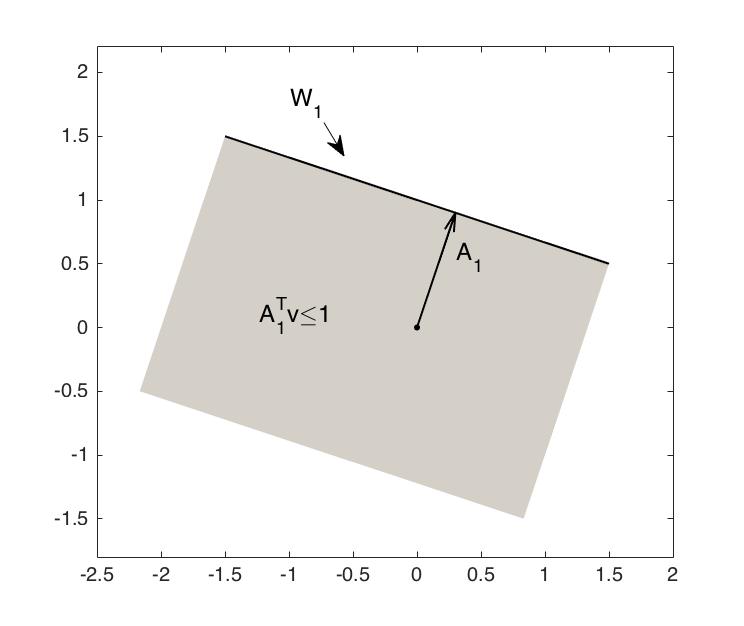}%
		\label{fig:1neuron intro}%
	}\qquad
	\subfloat[The effect of both the external charging and spikes on dual SNN.]{%
		\includegraphics[width=0.45\textwidth]{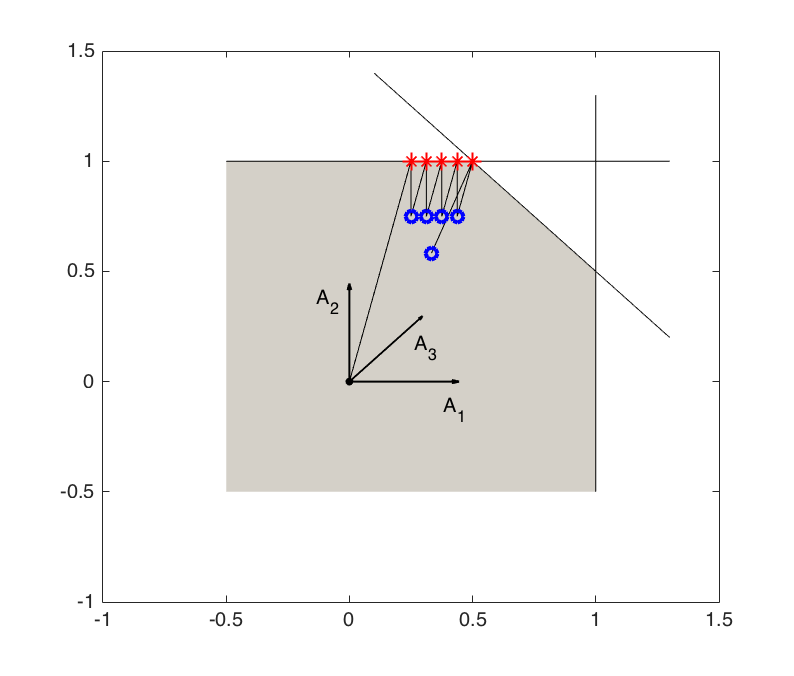}%
		\label{fig:dualSNN-2 intro}%
	}
	\caption{These are examples of the geometric interpretation of the dual SNN. In (a), we have one neuron where $A_1=[\frac{1}{2}\ 1]^\top$. In this case, neuron $i$ would not fire as long as the dual SNN $\bv(t)$ stays in the gray area.  In (b), we consider a SNN with 3 neurons where $A_1=[1\ 0]^\top$, $A_2=[0\ 1]^\top$, and $A_3=[\frac{2}{3}\ \frac{2}{3}]^\top$. One can see that the effect of spikes on dual SNN is a jump in the direction of the normal vector of the wall(s).}
\end{figure}
\vspace{3mm}

At each update in the dynamics, there are two terms affecting the dual SNN $\bv(t)$: the external charging $\bb\cdot dt$ and the spiking effect $-\alpha\cdot A\bs(t)$. First, one can see that the external charging $\bb\cdot dt$ can be thought of as a constant force that drags that dual SNN in the direction $\bb$.

To explain the effect of spikes in the dual view, let us start with an geometric view for the spiking rule. Recall that neuron $i$ fires a spike at time $t$ if and only if $\bu_i(t)>1$. In the language of dual SNN, this condition is equivalent to $A^\top_i\bv(t)>1$. Let $W_i=\{\bv\in\R^m:\ A_i^\top\bv=1 \}$ be the \textit{wall} of neuron $i$, the above observation is saying that neuron $i$ will fire a spike once it penetrates the wall $W_i$ from the half-space $\{\bv\in\R^m:\ A_i^\top\bv\leq1 \}$. See Figure~\ref{fig:1neuron intro} for an example. After neuron $i$ fires a spike, the spiking effect on the dual SNN $\bv(t)$ would be a $-\alpha\cdot A_i$ term, which corresponds to a jump in the \textit{normal direction} of $W_i$. See Figure~\ref{fig:dualSNN-2 intro} for an example.

The geometric interpretation described above is the main advantage of using dual SNN. Specifically, this gives us a clear picture of how spikes affect the SNN dynamics. Namely, neuron $i$ fires a spike if and only if the dual SNN $\bv(t)$ penetrates the $W_i$ and then $\bv(t)$ jumps back in the normal direction of $W_i$. Note that this connection would not hold in the primal SNN. In primal SNN $\bu(t)$, neuron $i$ fires a spike if and only if $\bu_i(t)>1$ while the effect on $\bu(t)$ is moving in the direction $-A^\top A_i$. See Table~\ref{table:geometric} for a comparison.

\begin{table}[h]
	\centering
	\begin{tabular}{|l|l|l|}
		\hline
		&&\\[-0.9em]
		& \multicolumn{1}{c|}{Primal SNN $\bu(t)$} & \multicolumn{1}{c|}{Dual SNN $\bv(t)$} \\[0.3em] \hline
		&&\\[-0.9em]
		Spiking rule   & $\bu_i(t)>1$                             & $A_i^\top\bv(t)>1$                     \\ [0.3em]\hline
		&&\\[-0.9em]
		Spiking effect & $-\alpha\cdot A^\top A_i$                & $-\alpha\cdot A_i$                      \\[0.3em] \hline
	\end{tabular}
	\caption{Comparison of the geometric view of primal and dual SNNs.}
	\label{table:geometric}
\end{table}

\paragraph{Dual view of SNN as a primal-dual algorithm for $\ell_1$ minimization problem}
First, let us write down the $\ell_1$ minimization problem and its dual. 

\vspace{3mm}
\begin{minipage}{\linewidth}
	\begin{minipage}{0.45\linewidth}
		\begin{equation*}
		\begin{aligned}
		& \underset{\bx\in\R^n}{\text{minimize}}
		& & \|\bx\|_1 \\
		& \text{subject to}
		& & A\bx=\bb.
		\end{aligned}
		\end{equation*}
	\end{minipage}
	\begin{minipage}{0.45\linewidth}
		\begin{equation*}
		\begin{aligned}
		& \underset{\bv\in\R^m}{\text{maximize}}
		& & \bb^{\top}\bv \\
		& \text{subject to}
		& & \|A^{\top}\bv\|_{\infty}\leq1.
		\end{aligned}
		\end{equation*}
	\end{minipage}
\end{minipage}
\vspace{3mm}

Now we observe that the dual dynamics can be viewed as a variant of the projected gradient descent algorithm to solve the dual program. Before the explanation, recall that for the $\ell_1$ minimization problem, we are considering the two-sided SNN for convenience. Indeed, without the spiking term, $\bv(t)$ simply moves towards the gradient direction $\bb$ of the dual objective function $\bb^\top\bv$. For the spike term $- \alpha \cdot A\bs(t)$, note that $\bs_i(t) \neq 0$ (\ie neuron $i$ fires) if and only if $|A^\top_i \bv(t)|=|\bu_i(t)| > 1$, which means that $\bv(t)$ is outside the feasible polytope $\{\bv:\ \|A^\top\bv\|_\infty\leq1\}$ of the dual program. Therefore, one can view the role of the spike term as \textit{projecting} $\bv(t)$ back to the feasible polytope. That is, when the dual SNN $\bv(t)$ becomes infeasible, it triggers some spikes, which maintains the dual feasibility and updates the primal solution (the firing rate). To sum up, we can interpret the simple SNN as performing a non-standard projected gradient descent algorithm for the dual program of $\ell_1$ minimization in the dual view of SNN.

With this primal-dual view in mind, we analyze the SNN algorithm by combining tools from convex geometry and perturbation theory as well as several non-trivial structural lemmas on the geometry of the dual program of $\ell_1$ minimization. One of the key ingredients here is identifying a \textit{potential function} that (i) upper bounds the error of solving $\ell_1$ minimization problem and (ii) monotonously converges to $0$. More details will be provided in Section~\ref{sec:SNN-simple-l1min}.

\subsection{Related Work}\label{sec:intro related}
We compare this research with other related works in the following four aspects.

\paragraph{Computational power of SNN}
Recognized as the third generation of neural networks~\cite{maass1997networks},the theoretical foundation for the computability of SNN had been built in the pioneering works of Maass et al.~\cite{maass1996lower,maass1997networks,maass19992,maass2001pulsed} in which SNN was shown to be able to simulate standard computational models such as Turing machines, random access machines (RAM), and threshold circuits. 

However, this line of works focused on the universality of the computational power and did not consider the efficiency of SNN in solving specific computational problems. 
In recent years, a line of exciting research have reported the efficiency of SNN in solving specific computational problems such as sparse coding~\cite{zylberberg2011sparse,tang2016convergence,tang2017sparse}, dictionary learning~\cite{lin2018dictionary}, pattern recognition~\cite{diehl2015unsupervised,kheradpisheh2016bio,bengio2017stdp}, and quadratic programming~\cite{barrett2013firing}. These works indicated the advantage of SNN in handling \textit{sparsity} as well as being \textit{energy efficient} and inspired real-world applications~\cite{beck2009fast}. However, to the best of our knowledge, no theoretical guarantee on the efficiency of SNN had been provided. 
For instance, Tang et al.~\cite{tang2016convergence,tang2017sparse} only proved the \textit{convergence in the limit} result for SNN solving sparse coding problem instead of giving an explicit convergence rate analysis.
The main contribution in this work is giving a rigorous guarantee on the convergence rate of the computational power of SNN.

\paragraph{The number of spikes versus the timing of spikes}
In this work, we mainly focused on the firing rate of SNN. That is, we only study the computational power with respect to the \textit{number} of spikes. Another important property of SNN is the \textit{timing} of spikes. 

The power of the timing of spikes had been reported since the 90s from some experimental evidences indicating that  neural systems might use the timing of spikes to encode information~\cite{abeles1991corticonics,hopfield1995pattern,rieke1999spikes}. From then on, a bunch of works have been focused on the aspect of time as a basis of information coding both from theoretical~\cite{olshausen1996emergence,maass1997networks,maass2001pulsed,thorpe2001spike} and experimental~\cite{heiligenberg1991neural,bialek1991reading,kuwabara1993delay} sides. It is generally believed that the timing of spikes is more powerful then the firing rate~\cite{thorpe1996speed,rullen2001rate,paugam2012computing}. Other than the capacity of encoding information, the timing of spikes has also been studied in the context of computational power~\cite{thorpe1996speed,maass1997fast,maass1997networks,gollisch2008rapid} and learning~\cite{booij2005gradient,banerjee2016learning,shrestha2017robust}. See the survey by Paugam et al.~\cite{paugam2012computing} for a thorough discussion.

While the timing of spikes is conceived as an important source of the power of SNN, in this work we simply focus on the firing rate and already yield some non-trivial findings in terms of the computational power.
We believe that our work is still in the very beginning stage of the study of the computational power of SNN.
Investigating how does the timing of spikes play a role is an interesting and important future direction. Immediate open questions here would be how could the timing of spikes fit into our study? What's the dual view of the timing of spikes? Can the timing of spikes solve the optimization problems more efficiently? Can the timing of spikes solve more difficult problems?

\paragraph{SNN with randomness}
While most of the literature focus on deterministic SNN, there is also an active line of works studying the SNN model with randomness\footnote{SNN model with noise is also known as stochastic SNN or noisy SNN depending on how the randomness involves in the model.}~\cite{allen1994evaluation,shadlen1994noise,faisal2008noise,buesing2011neural,jonke2014theoretical,maass2015spike,jonke2016solving,LMP17BDA,LMP17ITCS,LMP17DISC,LM18}.

Buesing et al.~\cite{buesing2011neural} used \textit{noisy SNN} to implement MCMC sampling and Jonke et al.~\cite{jonke2014theoretical,maass2015spike,jonke2016solving} further instantiated the idea to attack $\NP$-hard problems such as \textit{traveling salesman problem (TSP)} and \textit{constraint satisfaction problem (CSP)}. Concretely, their noisy SNN has a randomized spiking rule and the firing pattern would form a distribution over the solution space whereas the closer a solution is to the optimal solution, the higher the probability it is sampled. They got nice experimental performance in terms of solving empirical instance approximately. They also pointed out that their noisy SNN has the potential to be implemented energy-efficiently in practice.

Lynch, Musco, and Parter~\cite{LMP17ITCS} studied the \textit{stochastic SNNs} with a focus on the Winner-Take-All (WTA) problem. Their sequence of works~\cite{LMP17BDA,LMP17ITCS,LMP17DISC,LM18} gave the first asymptotic analysis for stochastic SNN in solving WTA, similarity testing, and neural coding. They view SNNs as distributed algorithms and derived computational tradeoff in running time and network size.

In this work, we consider the SNN model without randomness and thus is incomparable with the above SNN models with randomness. It is an interesting direction to apply the dual view of deterministic SNN to SNN with randomness.

\paragraph{Locally competitive algorithms}
Inspired by the dynamics of biological neural networks, Ruzell et al. designed the \textit{locally competitive algorithms} (LCA)~\cite{rozell2008sparse} for solving the Lasso (least absolute shrinkage and selection operator) optimization problem\footnote{Note that Lasso is equivalent to the Basis Pursuit De-Noising (BPDN) program under certain parameters transformation.}, which is widely used in statistical modeling. Roughly speaking, LCA is also a dynamics among a set of \textit{artificial neurons} which continuously signal their potential values (or a function of the values) to their neighboring neurons. There are two main differences between SNN and LCA. First, the neuron in SNN fires discrete spikes while the artificial neuron in LCA produces continuous signal. Next,  the neurons' potentials in LCA will converge to a fixed  value, which is the output of the algorithm. In contrast, in SNN, only the neurons' firing rates may converge instead of their potentials.

Nevertheless, there is a \textit{spikified} version of LCA introduced by Shapero et al.~\cite{shapero2013configurable,shapero2014optimal} called \textit{spike LCA (S-LCA)} in which the continuous signals are replaced with discrete spikes. S-LCA is almost the same as the SNN we are considering except a shrinkage term\footnote{That is, the potential of each neuron will drop with rate proportional to the current potential value.}.
Recently, Tang et al.~\cite{tang2017sparse} showed that the firing rate of S-LCA indeed converges to a variant of Lasso problem\footnote{In this variant, all the entries in matrix $A$ is non-negative.} in the limit. These works also experimentally demonstrated the efficient convergence of S-LCA and its advantage of fast identifying sparse solutions with potentially competitive practical performance to other Lasso algorithms (e.g., FISTA~\cite{beck2009fast}). However, there is no proof of convergence rate, and thus no explicit complexity bound of S-LCA.

\subsection{Future Works and Perspectives}\label{sec:future}
In this work, we give a theoretical study on the algorithmic power of SNN. Specifically, we focus on the firing rate of SNN and confirm an empirical analysis by Barrett et al.~\cite{barrett2013firing} with a convergence theorem (\textit{i.e.,} Theorem~\ref{thm:quadratic program informal}). Furthermore, we discover a dual view of SNN and show that SNN is able to solve the $\ell_1$ minimization problem (\textit{i.e.,} Theorem~\ref{thm: l1 informal}). In the following, we give interpretations to our results and point out future research directions.

First, how to interpret the dual dynamics of SNN? In this work, we discover the dual SNN based on mathematical convenience. Is there any biological interpretation?

Second, push further the analysis of simple SNN. We believe the parameters we get in the main theorems are not optimal. Is it possible to further sharpen the upper bound? We think this is both theoretically and practically interesting because both non-negative least squares and $\ell_1$ minimization are important problems that have been well-studied studied in the literature. Comparing the running time complexity or parallel time complexity of SNN algorithm with other algorithms could also be of theoretical interest and might inspire new algorithm with better complexity. Also, for practical purpose, having better parameters would give more confidence in implementing SNN as a natural algorithm.

Third, further investigate the potential of SNN dynamics as natural algorithms. The question is two-folded: (i) What algorithms can SNN implement? (ii) What computational problems can SNN solve? It seems that SNN is good at dealing with sparsity. Could it be helpful in related computational tasks such as fast Fourier transform (FFT) or sparse matrix-vector multiplication? It is interesting to identify optimization problems and class of instances where SNN algorithm can outperform other algorithms.

Finally, explore the practical advantage of SNN dynamics as natural algorithms. The potential practical time efficiency, energy efficiency, and simplicity for hardware implementation have been suggested in several works~\cite{mostafa2015event,binas2015spiking,biswas2016simple}. It would be exciting to see whether SNN has nice performance on practical applications such as compressed sensing, Lasso, and etc.

\section{Preliminaries}\label{sec:preliminaries}
In Section~\ref{sec:notations}, we build up some notations for the rest of the paper. In Section~\ref{sec:problems}, we define two optimization problems and the corresponding convergence guarantees.

\subsection{Notations}\label{sec:notations}
For any $n\in\N$, denote $[n]=\{1,2,\dots,n\}$ and $[\pm n]=\{\pm1,\pm2,\dots,\pm n\}$. Let $\bx,\by\in\R^n$ be two vectors. $|\bx|\in\R^n$ denotes the entry-wise absolute value of $\bx$, \ie $|\bx|_i=|\bx_i|$ for any $i\in[n]$. $\bx\preceq\by$ refers to entry-wise comparison, \ie $\bx_i\leq\by_i$ $\forall i\in[n]$. 

Let $A$ be an $m\times n$ real matrix. For any $i\in[n]$, denote the $i$th column of $A$ as $A_i$ and its negation to be $A_{-i}$, \ie $A_{-i}=-A_i$. When $A$ is positive semidefinite, we define the $A$-norm of a vector $\bx\in\R^n$ to be $\|\bx\|_A:=\sqrt{\bx^{\top}A\bx}$. Let $A^{\dagger}$ to be the pseudo-inverse of $A$. Define the maximum eigenvalue of $A$ as $\lambda_{\max}(A) := \max_{\bx\in\R^n:\ \|\bx\|_2=1}\|\bx\|_A$, the minimum non-zero eigenvalue of $A$ to be $\lambda_{\min}(A):=1/(\max_{\bx\in\R^n:\ \|\bx\|_2=1}\|\bx\|_{A^{\dagger}})$, and the condition number of $A$ to be $\kappa(A) := \lambda_{\max}(A)/\lambda_{\min}(A)$. If we do not specified, the following $\lambda_{\max},\lambda_{\min}$, and $\kappa$ are the eigenvalues and condition number of the connectivity matrix $C=\ATA$. For any $\bb\in\R^m$, we denote $\bb_A$ to be the projection of $\bb$ on the range space of $A$.

\subsection{Optimization problems}\label{sec:problems}
In this subsection, we are going to introduce two optimization problems: \textit{non-negative least squares} and \textit{$\ell_1$ minimization}.

\subsubsection{Non-negative least squares}
\begin{problem}[non-negative least squares]\label{prob:leastsquare}
	Let $m,n\in\N$. Given $A\in\R^{m\times n}$ and vector $\bb\in\R^m$, find $\bx\in\R^n$ that minimizes $\|\bb-A\bx\|_2^2/2$ subject to $\bx_i\geq0$ for all $i\in[n]$.
\end{problem}

\begin{remark}
	Recall that the least squares problem is defined as finding $\bx$ that minimize $\|\bb-A\bx\|_2$. That is, the non-negative least squares is a restricted version of the least squares problem. Nevertheless, one can use a non-negative least squares solver to solve the least squares problem by setting $A'=\bigl( \begin{smallmatrix}A^\top A & -A^\top A \\ -A^\top A & A^\top A\end{smallmatrix}\bigr)$ and $\bb'=\bigl( \begin{smallmatrix} \bb \\ -\bb\end{smallmatrix}\bigr)$ where $(A,\bb)$ is the instance of least squares and $(A',\bb')$ is the instance of non-negative least squares.
\end{remark}

The SNN algorithm might not solve the non-negative least squares problem exactly and thus we define the following notion of solving the non-negative least squares problem \textit{approximately}.

\begin{definition}[$\epsilon$-approximate solution to non-negative least squares]\label{def:non-negative least squares eps approx sol}
	Let $m,n\in\N$ and $\epsilon>0$. Given $A\in\R^{m\times n}$ and $\bb\in\R^m$. We say $\bx$ is an $\epsilon$-approximate solution to the non-negative least squares problem of $(A,\bb)$ if $\|A\bx-A\bx^*\|_2\leq\epsilon\|\bb\|_2$ where $\bx^*$ is an optimal solution.
\end{definition}

\subsubsection{$\ell_1$ minimization}
\begin{problem}[$\ell_1$ minimization]\label{prob:l1min}
	Let $m,n\in\N$. Given $A\in\R^{m\times n}$ and $\bb\in\R^m$ such that there exists a solution to $A\bx=\bb$. The goal of $\ell_1$ minimization is to solve the following optimization problem.
	\begin{equation*}
	\begin{aligned}
	& \underset{\bx\in\R^n}{\text{minimize}}
	& & \|\bx\|_1 \\
	& \text{subject to}
	& & A\bx=\bb.
	\end{aligned}
	\end{equation*}
\end{problem}

Similarly, we do not expect SNN algorithm to solve the $\ell_1$ minimization exactly. Thus, we define the notion of solving the $\ell_1$ minimization problem \textit{approximately} as follows.

\begin{definition}[$\epsilon$-approximate solution to $\ell_1$ minimization]\label{def:l1 min eps approx sol}
	Let $m,n\in\N$ and $\epsilon>0$. Given $A\in\R^{m\times n}$ and $\bb\in\R^m$. Let $\OPT^{\ell_1}$ denote the optimal value of the $\ell_1$ minimization problem of $(A,\bb)$. We say $\bx\in\R^n$ is an $\epsilon$-approximate solution of the $\ell_1$ minimization problem of $(A,\bb)$ if $\|\bb-A\bx\|_2\leq\epsilon\cdot\|\bb\|_2$ and $\|\bx\|_1-\OPT^{\ell_1}\leq\epsilon\cdot\OPT^{\ell_1}$.
\end{definition}

\subsection{Karush-Kuhn-Tucker conditions}\label{sec:KKT}
Karush-Kuhn-Tucker (KKT) conditions are necessary and sufficient conditions for the optimality of optimization problems under some regular assumptions. Consider the following optimization program.
\begin{equation}\label{op:convex}
\begin{aligned}
& \underset{\bx\in\R^n}{\text{minimize}}
& & f(\bx) \\
& \text{subject to}
& & g_i(\bx)\leq0,&\forall i=1,2,\dots m,\\
& & & h_j(\bx)=0,&\forall j=1,2,\dots,k,
\end{aligned}
\end{equation}
where $f,g_1,\dots,g_m,h_1,\dots,h_k$ are convex and differentiable. Let $\bv\in\R^m$ and $\bmu\in\R^k$ be the dual variables. KKT conditions give necessary and sufficient conditions for $(\bx,\bv,\bmu)$ be a pair of primal and dual optimal solutions.

\begin{theorem}[KKT conditions]\label{thm:KKT}
	$(\bx,\bv,\bmu)$ are a pair of primal and dual optimal solutions for~\eqref{op:convex} if and only if the following conditions hold.
	\begin{itemize}
		\item $\bx$ is primal feasible, \textit{i.e.,} $g_i(\bx)\leq$ and $h_j(\bx)=0$ for all $i\in[m]$ and $j\in[k]$.
		\item $(\bv,\bmu)$ is dual feasible, \textit{i.e.,} $\bv_i\geq0$ for all $i\in[m]$.
		\item The Lagrange multiplier vanishes, \textit{i.e.,} $\nabla f(\bx) + \sum_{i\in[m]}\bv_i\nabla g_i(\bx) + \sum_{j\in[k]}\bmu_j\nabla h_j(\bx)=0$.
		\item $(\bx,\bv,\bmu)$ satisfy complementary slackness, \textit{i.e.,} $\bv_i f_i(\bx)\geq0$ for all $i\in[m]$.
	\end{itemize}
\end{theorem}

For more details about KKT conditions, please refer to standard textbook such as Chapter 5.5.3 in~\cite{boyd2004convex}.

\subsection{Perturbation theory}\label{sec:perturbation}
Perturbation theory, sometimes known as sensitivity analysis, for optimization problems concerns the situation where the optimization program is perturbed and the goal is to give a good estimation for the optimal solution. See a nice survey by Bonnans and Shapiro~\cite{bonnans1998optimization}. In the following we state a special case for convex optimization program with strong duality.

\begin{theorem}[perturbation, Chapter 5.6 in~\cite{boyd2004convex}\footnote{Note that we switch the original and perturbed programs in the statement in~\cite{boyd2004convex}.}]\label{thm:perturbation}
	Given the following two optimization programs where the strong duality holds and there exists feasible dual solution.\\
	\begin{minipage}{\linewidth}
		\begin{minipage}{0.5\linewidth}
			\begin{equation}\label{op:prelim perturbation-original}
			\begin{aligned}
			& \underset{\bx}{\text{minimize}}
			& & f(\bx) \\
			& \text{subject to}
			& & g_i(\bx)\leq0,&\forall i=1,2,\dots,m,\\
			& & & h_j(\bx)=0,&\forall j=1,2,\dots,k.
			\end{aligned}
			\end{equation}
		\end{minipage}
		\begin{minipage}{0.5\linewidth}
			\begin{equation}\label{op:prelim perturbation-perturbed}
			\begin{aligned}
			& \underset{\bx}{\text{minimize}}
			& & f(\bx) \\
			& \text{subject to}
			& & g_i(\bx)\leq \mathbf{a}_i,&\forall i=1,2,\dots,m,\\
			& & & h_j(\bx)=\mathbf{b}_j,&\forall j=1,2,\dots,k.
			\end{aligned}
			\end{equation}
		\end{minipage}
	\end{minipage}
	
	\vspace{1em}
	
	Let $\OPT^{\text{original}}$ be the optimal value of the original program~\eqref{op:prelim perturbation-original} and $\OPT^{\text{perturbed}}$ be the optimal value of the perturbed program~\eqref{op:prelim perturbation-perturbed}. Let $(\bv^*,\bmu^*)\in\R^m\times\R^k$ be the optimal dual solution of the perturbed program~\eqref{op:prelim perturbation-perturbed}. We have
	\begin{equation*}
	\OPT^{\text{original}}\geq\OPT^{\text{perturbed}}+\mathbf{a}^\top\bv^*+\mathbf{b}^\top\bmu^*.
	\end{equation*}
\end{theorem}

\section{A simple SNN algorithm for $\ell_1$ minimization}\label{sec:SNN-simple-l1min}
In this section, we focus on the discovery of the dual view of simple SNN and how it can be viewed as a \emph{primal-dual algorithm} for solving the $\ell_1$ minimization problem.

Recall that for the $\ell_1$ minimization problem, we are working on the \textit{two-sided} simple SNN for the convenience of future analysis. That is,
$$
\frac{d}{dt}\bu(t) = -\alpha\cdot A^\top A\bs(t) +A^\top\bb,
$$
where $\bs_i(t)=\mathbf{1}_{\bu_i(t)>\eta}-\mathbf{1}_{\bu_i(t)<-\eta}$. To solve the $\ell_1$ minimization problem, we configure a two-sided simple SNN as follows. Given an input $(A,\bb)$, let $C=A^\top A$ and $\bI=A^\top\bb$. However, currently it is unclear how does the above simple SNN dynamics relate to the $\ell_1$ minimization problem.
\begin{equation}\label{op:basispursuit}
\begin{aligned}
& \underset{\bx\in\R^n}{\text{minimize}}
& & \|\bx\|_1 \\
& \text{subject to}
& & A\bx=\bb.
\end{aligned}
\end{equation}

Interesting, the connection between simple SNN and the $\ell_1$ minimization problem happens in the \textit{dual program} of the $\ell_1$ minimization problem. Before we formally explain this connection, let us write down the dual program of~\eqref{op:basispursuit}.
\begin{equation}\label{op:basispursuit-dual}
\begin{aligned}
& \underset{\bv\in\R^m}{\text{maximize}}
& & \bb^{\top}\bv \\
& \text{subject to}
& & \|A^{\top}\bv\|_{\infty}\leq1.
\end{aligned}
\end{equation}

Let us try to make some geometric observations on~\eqref{op:basispursuit-dual}. First, the objective of the dual program is to maximize the inner product with $\bb$, which is quite related to the external charging of SNN since we take $\bI=A^\top\bb$. Next, the feasible region of the dual program is a polytope (or a polyhedron) defined by the intersection of half-spaces $\{\bv\in\R^m:\ A_i^\top\bv\leq1\}$ and $\{\bv\in\R^m:\ -A_i^\top\bv\leq1\}$ for each $i\in[n]$ where $A_i$ denotes the $i^\text{th}$ column of $A$.

It will be convenient to introduce the following notation before we move on. For $i\in [n]$, let $A_{-i} = -A_i$. Let $[\pm n] = \{\pm1,\pm2,\dots,\pm n\}$. Thus, the feasible polytope of the dual program is defined by the intersection of half-spaces defined by $A^\top_j \bv \leq 1$ for all $j\in [\pm n]$. We call this polytope the \textit{dual polytope}\footnote{In the case where the feasible region of the dual program is not bounded, it is a dual \textit{polyhedron}. For the convenience of the presentation, we usually assume the feasible region is bounded.}. Moreover, for each $j\in [\pm n]$, we call the hyperplane $ \{\bv: A^\top_j \bv = 1 \}$ the \emph{wall} $W_j$ of the dual polytope. See Figure~\ref{fig:3neuron} for examples.

\begin{figure}[h]
	\centering
	\subfloat[An example of one neuron.]{%
		\includegraphics[width=0.45\textwidth]{pic/1neuron.jpg}%
		\label{fig:1neuron}%
	}\qquad
	\subfloat[An example of three neurons.]{%
		\includegraphics[width=0.45\textwidth]{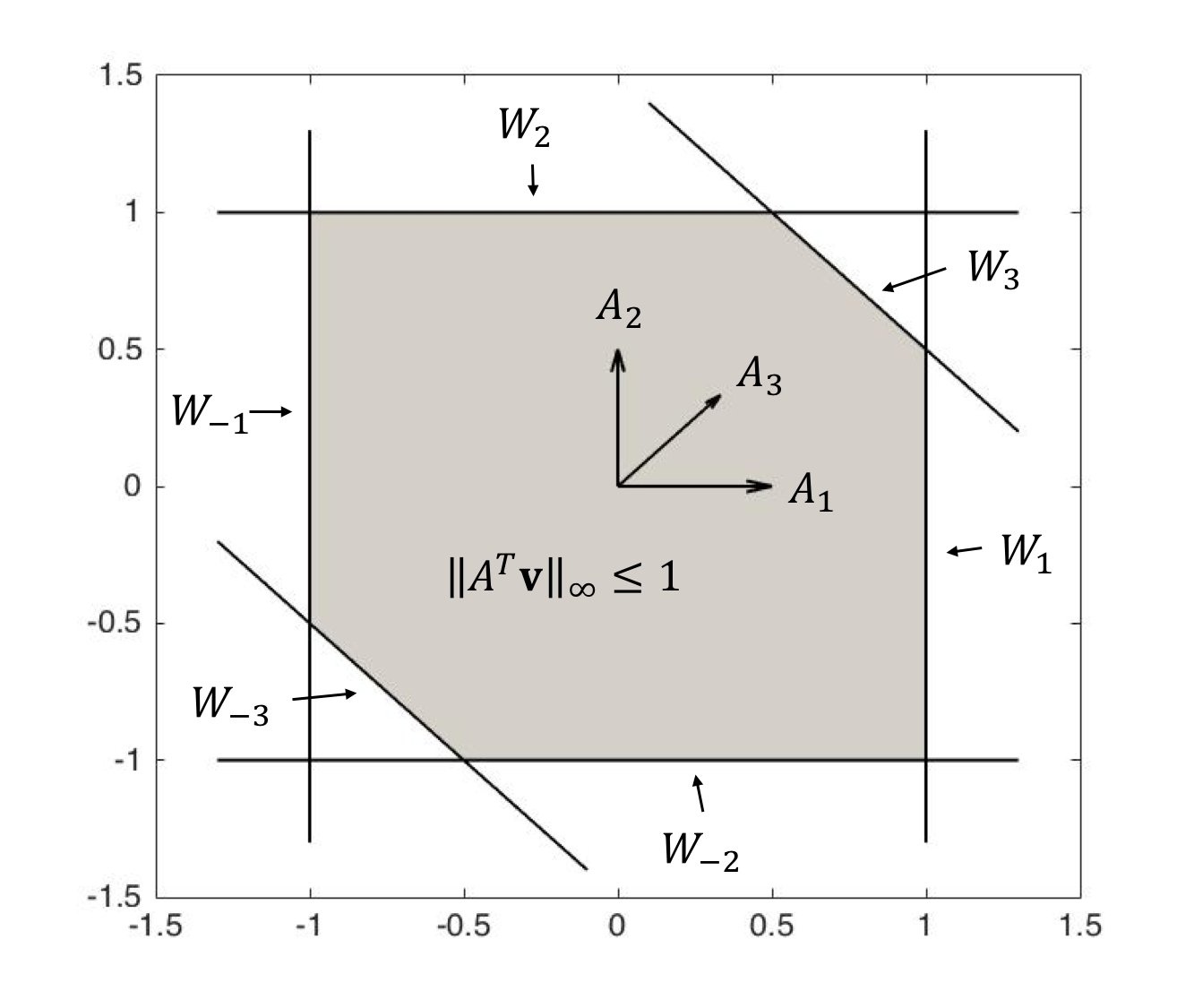}%
		\label{fig:3neuron}%
	}
	\caption{This is examples of the geometric interpretation of the dual program of $\ell_1$ minimization problem. In (a), we have $n=1$ where $A_1=[\frac{1}{3}\ 1]^\top$. In this case, the gray area, \textit{i.e.,} the feasible region of the dual program, is unbounded. In (b), we have $n=3$ where $A_1=[1\ 0]^\top$, $A_2=[0\ 1]^\top$, and $A_3=[\frac{2}{3}\ \frac{2}{3}]^\top$. In this case, the gray area is bounded and thus called dual polytope.}
\end{figure}
\vspace{3mm}

Now, the key observation is that by a linear transformation, the dynamics of simple SNN has a natural interpretation in the dual space. We call it the \emph{dual SNN} defined as follows.

\subsection{Dual SNN} \label{sec:dual SNN}
We first recall the simple SNN dynamics which we call the \textit{primal SNN} from now on. For convenience, we set the threshold parameter $\eta = 1$ (and make the spiking strength parameter $\alpha$ explicit). For any $t\geq0$,
\begin{equation}\label{eq:l1 SNNdynamics}
\bu(t+dt) = \bu(t)-  \alpha \cdot \ATA \cdot \bs(t)+ \ATb \cdot dt.
\end{equation}
Now, we define the dual SNN $\bv(t) \in \R^m$ as follows. Let $\bv(0)=\mathbf{0}$ and for each $t\geq0$, define
\begin{equation}\label{eq:l1 SNNdynamics-dual}
\bv(t+dt) = \bv(t) - \alpha \cdot A\bs(t) + \bb\cdot dt.
\end{equation}
Let us make some remarks about the connection between the primal and dual SNNs. First, it can be immediately seen that $\bu(t)=A^\top\bv(t)$ for each $t\in\N$ from~\eqref{eq:l1 SNNdynamics} and~\eqref{eq:l1 SNNdynamics-dual}. That is, given $\bv(t)$, it is easy to get $\bu(t)$ by multiplying $\bu(t)$ with $A^\top$ on the left. It turns out that the other direction also holds. For each $t\in\N$, we have $\bv(t) = (A^\top )^{\dagger}\bu(t)$, where $(A^\top)^{\dagger}$ is the pseudo-inverse of $A^\top$. The reason is because the primal SNN $\bu(t)$ lies in the column space of $A$. Thus, the two dynamics are in fact \emph{isomorphic} to each other.

Now let us understand the dynamics of dual SNN in the dual space $\R^m$. At each timestep, there are two terms, \textit{i.e.,} the external charging $\bb\cdot dt$ and the spiking effect $-\alpha A\bs(t)$, that affect the dual SNN $\bv(t)$. The external charging can be thought of as a constant force that drags that dual SNN in the direction $\bb$. See Figure~\ref{fig:dualSNN-1}. This coincides with the objective function of the dual program~\eqref{op:basispursuit-dual} and thus the external charging can then be viewed as taking a \textit{gradient} step towards solving~\eqref{op:basispursuit-dual}.

Nevertheless, to solve~\eqref{op:basispursuit-dual}, one need to make sure the solution $\bv$ is feasible, \textit{i.e.,} $\bv$ should lie in the dual polytope. Interestingly, this is exactly what the spike is doing!
Recall that neuron $i$ fires a spike if $|u_i(t)| > 1$ (recall that we set $\eta = 1$), which corresponds to $|A^\top_i \bv(t)| > 1$ in the dual space. Thus, the spike term has the following nice geometric interpretation: if $\bv(t)$ ``exceeds'' the wall $W_j$ for some $j\in [\pm n]$, then neuron $|j|$ fires a spike and $\bv(t)$ is ``bounced back'' in the normal direction of the wall $W_j$ in the sense that $\bv(t)$ is subtracted by $\alpha \cdot A_j$. See Figure~\ref{fig:dualSNN-2} for example.

\begin{figure}[h]
	\centering
	\subfloat[The effect of external charging on dual SNN.]{%
		\includegraphics[width=0.45\textwidth]{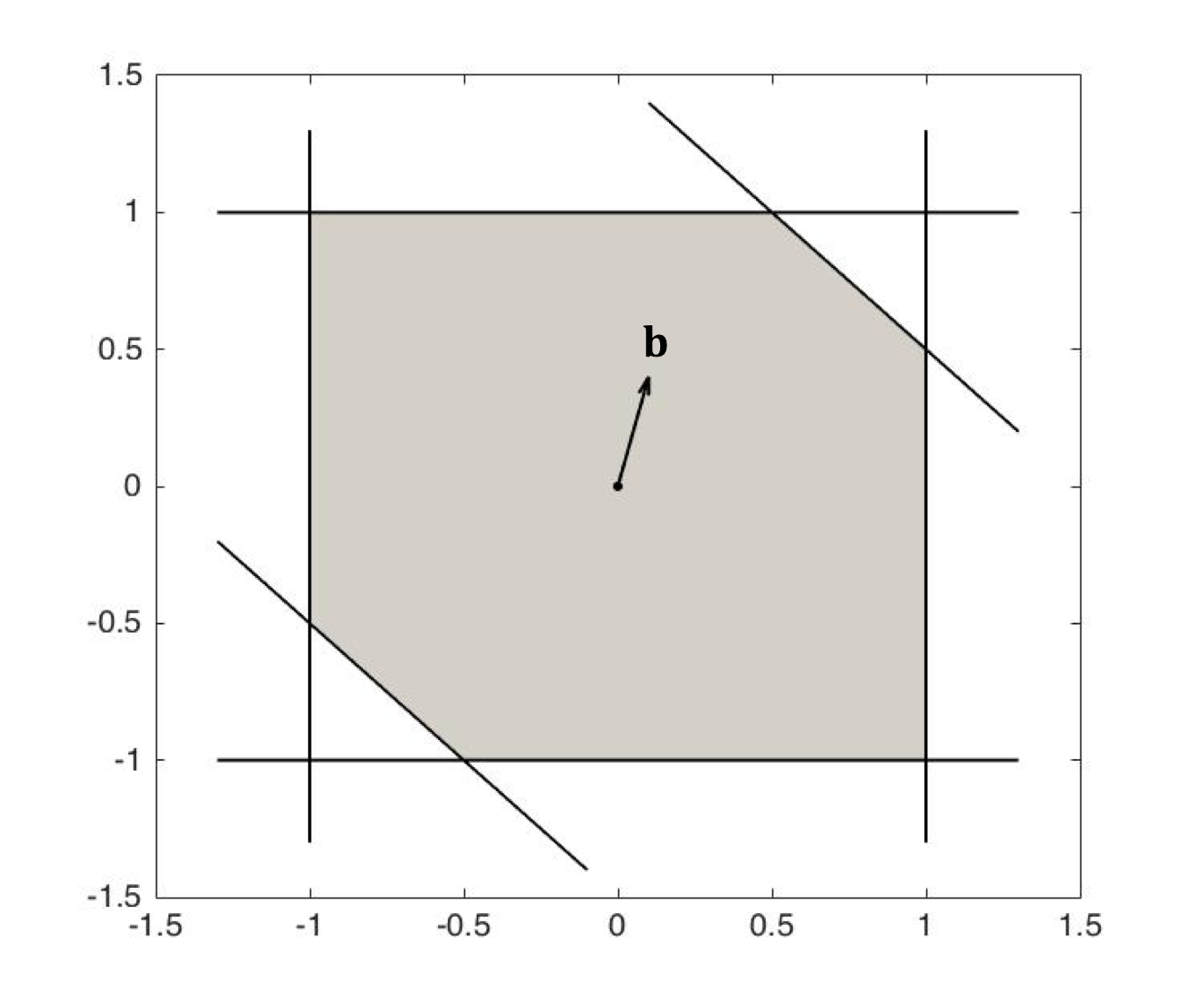}%
		\label{fig:dualSNN-1}%
	}\qquad
	\subfloat[The effect of both the external charging and spikes on dual SNN.]{%
		\includegraphics[width=0.45\textwidth]{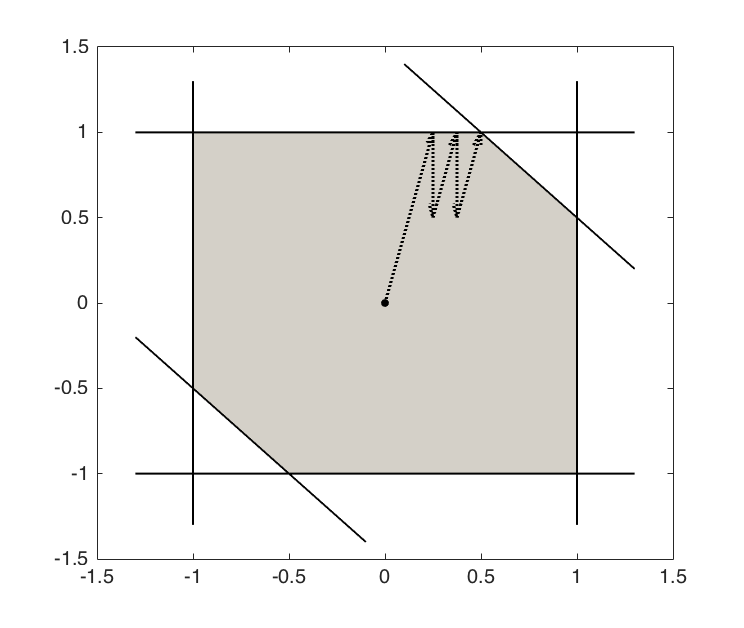}%
		\label{fig:dualSNN-2}%
	}
	\caption{This is examples of the geometric interpretation of the dual We consider the same matrix $A$ as in Figure~\ref{fig:3neuron} and $\bb=[0.1\ 0.4]^\top$. In (a), one can see that the external charging $\bb$ points the direction that dual SNN is moving. In (b), one can see that the effect of spikes on dual SNN is a jump in the direction of the normal vector of the wall.}
\end{figure}
\vspace{3mm}

Therefore, one can view the dual SNN as performing a variant of projected gradient descent algorithm for the dual program of $\ell_1$ minimization problem. Specifically, to maintain the feasibility, the vector is not projected back to the feasible region as usual, but is ``bounced back'' in the normal direction of the wall $W_j$ corresponding to the violated constraint $A^\top_j \bv \leq 1$. An advantage of this variant is that the ``bounced back'' operation is simply subtraction of $\alpha \cdot A_j$, which is significantly more efficient than the orthogonal projection back to the feasible region. 
On the other hand, note that the dynamics might not exactly converge to the optimal dual solution $\bv^{\OPT}$. Intuitively, the best we can hope for is that $\bv(t)$ will converge to a small neighboring region of $\bv^{\OPT}$(assuming the spiking strength $\alpha$ is sufficiently small). The above intuition of viewing dual SNN as a projected gradient descent algorithm for the dual program of the $\ell_1$-minimization problem will be formally proved in the later subsections.

\paragraph{The primal-dual connection.} So far we have informally seen that the dual SNN can be viewed as solving the dual program of the $\ell_1$-minimization problem. However, this does not immediately give us a reason why the firing rate would converge to the solution of the primal program. It turns out that there is a beautiful connection between the dual SNN and firing rate through the \textit{Karush-Kuhn-Tucker (KKT) conditions} (see Section~\ref{sec:KKT}) and perturbation theory (see Section~\ref{sec:perturbation}).

We now discuss some intuitions about how the dual solution translates to the primal solution. To jump into the core idea, let us consider an ideal scenario where the dual SNN $\bv(t)$ is already very close to the optimal dual solution $\bv^\OPT$ for the dual program of the $\ell_1$ minimization problem. Since $\bv^\OPT$ is the optimal solution and thus it must lie on the boundary of the dual polytope. Let $\Gamma\subseteq[\pm n]$ be the set of walls that $\bv^\OPT$ touches. That is, $j\in\Gamma$ if and only if $A_j^\top\bv^\OPT=1$. Now, let $\bx^\OPT$ denote the optimal primal solution of the $\ell_1$ minimization problem. Observe that by the complementary slackness in the KKT conditions, for each $i\in[n]$, we have $\bx^\OPT_i>0$ (resp. $\bx^\OPT_i<0$) if $i\in\Gamma$ (resp. $-i\in\Gamma$) and $\bx^\OPT_i=0$ if $i,-i\not\in\Gamma$. To summary, this is saying that $\Gamma$ contains the coordinates that are non-zero in the primal optimal solution $\bx^\OPT$. See Figure~\ref{fig:primaldual} for an example.

\begin{figure}[h]
	\centering
	\includegraphics[width=7cm]{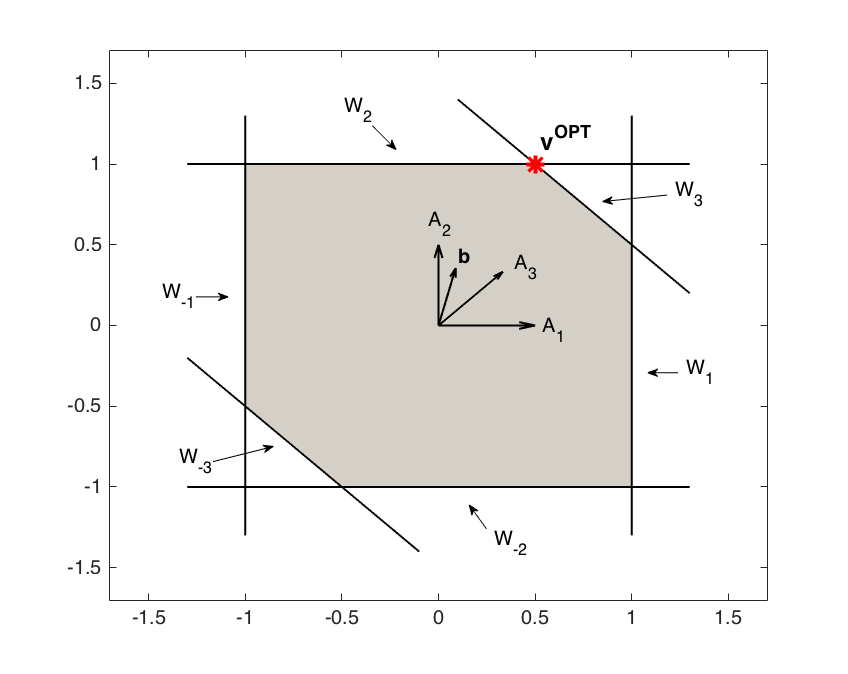}
	\caption{This is an example based on Figure~\ref{fig:3neuron} and Figure~\ref{fig:dualSNN-2}. In this example, $A_1=[1\ 0]^\top$, $A_2=[0\ 1]^\top$, $A_3=[\frac{2}{3}\ \frac{2}{3}]^\top$, and $\bb=[0.1\ 0.4]^\top$. The optimal dual solution is $\bv^\OPT=[\frac{1}{2}\ 1]^\top$ as shown in the figure. Thus, by the above definition we have $\Gamma=\{2,3\}$. By the KKT conditions, we then know that only the $2^\text{nd}$ and $3^{\text{rd}}$ coordinate of the optimal primal solution are non-zero. Indeed, the optimal primal solution is $\bx^\OPT=[0\ \frac{3}{10}\ \frac{3}{20}]^\top$.}
	\label{fig:primaldual}
\end{figure}

With this observation, once the dual SNN $\bv(t)$ is very close to the optimal dual solution $\bv^\OPT$ and stays nearby, only those neurons correspond to $\Gamma$ would fire spikes. In other words, the firing rate of the non-zero coordinates in the primal optimal solution $\bx^\OPT$ will remain non-zero due to the spikes while the other coordinates will gradually go to zero.

At this point, we have seen that (i) the dual SNN can be viewed as a projected gradient descent algorithm for the dual program of $\ell_1$ minimization problem and (ii) the dual solution (resp. dual SNN) and primal solution (resp. firing rate) have a natural connection through the KKT conditions. The explanations so far are rather informal and focus on intuition. From now on, everything will start to be more and more formal and rigorous. Before that, let us state the main theorem of this section about simple SNN solving $\ell_1$ minimization problem.

\begin{theorem}\label{thm:l1}
	Given $A\in\R^{m\times n}$ and $\bb\in\R^m$ where all the row of $A$ has unit norm. Let $\gamma(A)$ be the niceness parameter of $A$ defined later in Definition~\ref{def:nice}. Suppose $\gamma(A)>0$ and there exists a solution for $A\bx=\bb$. There exists a polynomial $\alpha(\cdot)$ such that for any $t\geq0$, let $\bx(t)$ be the firing rate of the simple SNN with $C = \ATA$, $\bI = \ATb$, $\eta=1$, $0<\alpha\leq \alpha(\frac{\gamma(A)}{n\cdot\lambda_{\max}})$. Let $\OPT^{\ell_1}$ be the optimal value of the $\ell_1$ minimization problem. For any $\epsilon>0$, when $t\geq\Omega(\frac{m^2\cdot n\cdot\|\bb\|_2^2}{\epsilon^2\cdot\lambda_{\min}\cdot\OPT^{\ell_1}})$, then $\bx(t)$ is an $\epsilon$-approximate solution for the $\ell_1$ minimization problem for $(A,\bb)$.
\end{theorem}

Two remarks on the statement of Theorem~\ref{thm:l1}. First, we consider the \textit{continuous SNN} instead of the discrete SNN, which is of interest for simulation on classical computer. In discrete SNN, the \textit{step size} is some non-negligible $\Dt>0$ instead of $dt$. The main reason for considering continuous SNN is that this significantly simplify the proof by avoiding a huge amount of nasty calculations. We suspect that the proof idea would hold for discrete SNN with discretization parameter $\Dt\leq\Dt(\frac{\gamma(A)}{n\cdot\lambda_{\max}})$ for some polynomial $\Dt(\cdot)$.
Second, the parameters in Theorem~\ref{thm:l1} have not been optimized and we believe all the dependencies can be improved. Since the parameters highly affect the efficiency of SNN as an algorithm for $\ell_1$ minimization problem, we pose it as an interesting open problem to study what are the best dependencies one can get.

\subsection{Overview of the proof for Theorem~\ref{thm:l1}}

The proof of Theorem~\ref{thm:l1} consists of two main steps as mentioned in the previous subsection. The first step argues that the dual SNN $\bv(t)$ would converge to the neighborhood of the optimal dual solution $\bv^\OPT$. The second step is connecting the dual solution (\textit{i.e.,} the dual SNN) to the primal solution (\textit{i.e.,} the firing rate).

In the first step, we try to identify a \textit{potential function}\footnote{Potential function is widely used in the analysis of many gradient-descent based algorithm. The difficulty lies in the search of a good potential function for the algorithm.} that captures how close is $\bv(t)$ to the optimal dual solution $\bv^\OPT$.
It turns out that this is not an easy task since the effect of spikes makes the behavior of dual SNN very non-monotone. We conquer the difficulty via a technique that we call \textit{ideal coupling} (see Definition~\ref{def:ideal coupling} and Figure~\ref{fig:ideal coupling}). The idea is associating the dual SNN $\bv(t)$ with an \textit{ideal SNN} $\bv^\text{ideal}(t)$ for every $t\geq0$ such that the ideal SNN would have \textit{smoother} behavior comparing to the spiking phenomenon in the dual SNN. We will formally define the ideal SNN in Section~\ref{sec:ideal coupling}. There are two advantages of using ideal SNN instead of handling dual SNN directly: 
(i) Ideal SNN is smoother than dual SNN in the sense that it would not change after spikes (see Lemma~\ref{lem:ideal SNN unchaged}). Further, by introducing some auxiliary processes (\textit{i.e.,} the auxiliary SNNs defined in Definition~\ref{def:auxiliary}), we are able to identify a potential function that is strictly improving at any moment and measures how well the dual SNN has been solving the $\ell_1$ minimization problem (see Lemma~\ref{lem:strict improvement}). 
(ii) ideal SNN is naturally associated with an \textit{ideal solution} (defined in Definition~\ref{def:ideal solution}) which is easier to analyze than the firing rate. Using these good properties of ideal SNN, we can prove in Lemma~\ref{lemma:idealalgorithm-l2bound} that the $\ell_2$ residual error of the ideal solution will converge to $0$.

After we are able to show the convergence of the $\ell_2$ residual error in Lemma~\ref{lemma:idealalgorithm-l2bound}, we move to the second step where the goal is showing that the $\ell_1$ norm of the solution is also small. We look at the KKT conditions of the $\ell_1$ minimization problem and observe that the primal and dual solutions of SNN satisfy the KKT conditions of a \textit{perturbed} program of the $\ell_1$ minimization problem. Finally, combine tools from perturbation theory, we can upper bound the $\ell_1$ error of the ideal solution by its $\ell_2$ residual error in Lemma~\ref{lemma:idealalgorithm-OPTbounds}.

Theorem~\ref{thm:l1} then follows from Lemma~\ref{lemma:idealalgorithm-l2bound} and Lemma~\ref{lemma:idealalgorithm-OPTbounds} with some special cares on how to transform everything for ideal solution to the firing rate. See Figure~\ref{fig:overview proof for thm l1} for an overall structure of the proof for Theorem~\ref{thm:l1}.

\begin{figure}[h]
	\centering
	\begin{tikzpicture}[
	squarednode/.style={rectangle, draw=black!100, fill=red!5, very thick, minimum size=5mm},
	defnode/.style={rectangle, draw=black!100, fill=blue!5, very thick, minimum size=5mm},
	toolnode/.style={rectangle, draw=black!100, fill=green!5, very thick, minimum size=5mm},
	]
	\tikzstyle{line} = [draw, thick, color=black!50, -latex']
	\path node (thml1) [squarednode] {$\stackrel{\text{\autoref{thm:l1}}}{\text{(SNN solves $\ell_1$ minimization problem)}}$};
	\path (thml1.north) + (-4.0,1.1) node (l2UB) [squarednode] {$\stackrel{\text{Lemma~\ref{lemma:idealalgorithm-l2bound}}}{\text{(convergence of $\ell_2$ error)}}$};
	\path (thml1.north) + (4.0,1.1) node (l2boundl1) [squarednode] {$\stackrel{\text{Lemma~\ref{lemma:idealalgorithm-OPTbounds}}}{\text{($\ell_2$ error upper bounds $\ell_1$ error)}}$};
	\path (l2boundl1.north) + (-1.6,1.0) node (KKT) [toolnode] {KKT conditions};
	\path (l2boundl1.north) + (1.4,1.0) node (perturbation) [toolnode] {Perturbation};
	\path (l2UB.north) + (-2.8,1.0) node (wouldnotchange) [squarednode] {$\stackrel{\text{Lemma~\ref{lem:ideal SNN unchaged}}}{\text{(unchaged after spikes)}}$};
	\path (l2UB.north) + (2.8,1.0) node (strict) [squarednode] {$\stackrel{\text{Lemma~\ref{lem:strict improvement}}}{\text{(strict improvement)}}$};
	\path (wouldnotchange.north) + (0,1.0) node (nice) [defnode] {$\stackrel{\text{Definition~\ref{def:nice}}}{\text{(niceness of input matrix)}}$};
	\path (strict.north) + (-1.6,1.0) node (ideal) [defnode] {$\stackrel{\text{Definition~\ref{def:ideal coupling}}}{\text{(ideal coupling)}}$};
	\path (strict.north) + (1.6,1.0) node (auxiliary) [defnode] {$\stackrel{\text{Definition~\ref{def:auxiliary}}}{\text{(auxiliary SNN)}}$};
	
	\path [line] (l2UB.south) -- +(0.0,-0.3) -- +(4,-0.3)
	-- node [above, midway] {} (thml1);
	\path [line] (l2boundl1.south) -- +(0.0,-0.3) -- +(-4,-0.3)
	-- node [above, midway] {} (thml1);
	\path [line] (KKT.south) -- +(0.0,-0.4) -- +(1.6,-0.4)
	-- node [above, midway] {} (l2boundl1);
	\path [line] (perturbation.south) -- +(0.0,-0.4) -- +(-1.4,-0.4)
	-- node [above, midway] {} (l2boundl1);
	\path [line] (wouldnotchange.south) -- +(0.0,-0.2) -- +(2.8,-0.2)
	-- node [above, midway] {} (l2UB);
	\path [line] (strict.south) -- +(0.0,-0.2) -- +(-2.8,-0.2)
	-- node [above, midway] {} (l2UB);
	\path [line] (ideal.south) -- +(0.0,-0.2) -- +(1.6,-0.2)
	-- node [above, midway] {} (strict);
	\path [line] (auxiliary.south) -- +(0.0,-0.2) -- +(-1.6,-0.2)
	-- node [above, midway] {} (strict);
	\path [line] (nice.south) -- +(0.0,0) 
	-- node [above] {} (wouldnotchange);
	
	\end{tikzpicture}
	\caption{Overview of the proof for Theorem~\ref{thm:l1}.}
	\label{fig:overview proof for thm l1}
\end{figure}
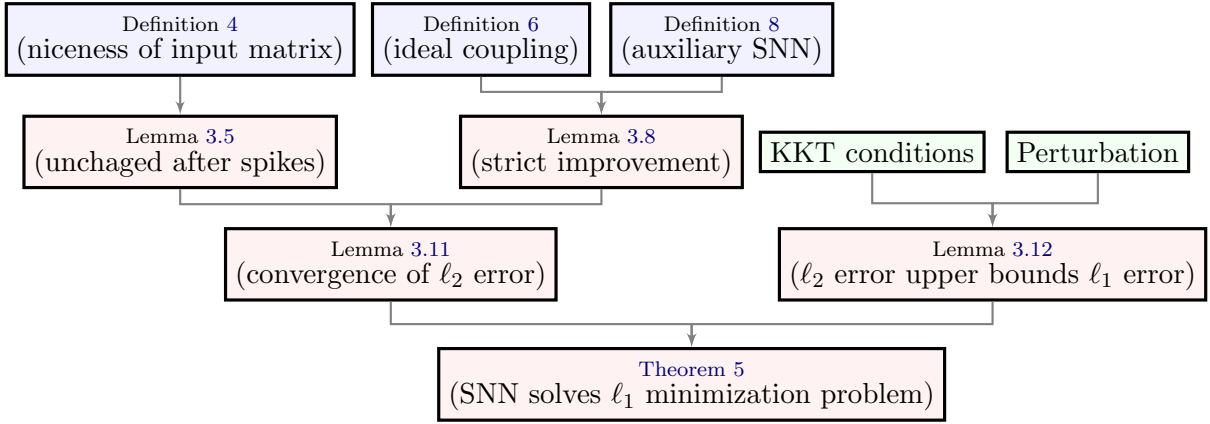

In the rest of this section, we are going to start from some definitions on the \textit{nice conditions} we need for the input matrix in Section~\ref{sec:nice}. Next, we define the ideal coupling in Section~\ref{sec:ideal coupling} and prove Lemma~\ref{lem:ideal SNN unchaged} and Lemma~\ref{lem:strict improvement} in Section~\ref{sec:unchange} and Section~\ref{sec:strict} respectively. Finally, we wrap up the proof for Theorem~\ref{thm:l1} in Section~\ref{sec:l1 convergence}.

\subsection{Some nice conditions on the input matrix}\label{sec:nice}
We need some \textit{nice conditions} for the input matrix as follows.
\begin{definition}[non-degeneracy]
	Let $A\in\R^{m\times n}$ where $m\leq n$. We say $A$ is non-degenerate if for any size $m\times m$ submatrix of $A$ has full rank. For any $\gamma>0$, we say $A$ is $\gamma$-non-degenerate if for any $\Gamma\subseteq[n]$, $|\Gamma|=m$, and $i\in\Gamma$, $\|A_i-\Pi_{A_{\Gamma\backslash\{i\}}}A_i\|_2\geq\gamma$ where $\Pi_{A_{\Gamma\backslash\{i\}}}\bv$ is the projection of $\bv$ onto subspace spanned by $\{A_j:\ j\in\Gamma\backslash\{i\}\|\}$ for any $\bv\in\R^m$.
\end{definition}

Note that if $A$ is non-degenerate, then for any $S\subseteq[n]$ and $|S|=m$ and $\bb\in\{-1,1\}^m$, there exists an unique solution $\bv\in\R^m$ to $A_S^\top\bv=\bb$ where $A_S$ is the submatrix of $A$ restricted to columns in $S$. We call such $\bv$ a \textit{vertex} of the polytope $\mathcal{P}_{A,1}$. Note that in this definition, a vertex might not lie in $\mathcal{P}_{A,1}$. An important parameter for future analysis is the minimum distance between two distinct vertices of $\mathcal{P}_{A,1}$.

\begin{definition}[nice input matrix]\label{def:nice}
	Let $A\in\R^{m\times n}$ and $\gamma\geq0$. We say $A$ is $\gamma$-nice if all of the following conditions hold.
	\begin{enumerate}[label=(\arabic*)]
		\item $A$ is $\gamma$-non-degenerate.
		\item The distance between any two distinct vertices of $\mathcal{P}_{A,1}$ is at least $\gamma$.
		\item For any $\bb\in\{-1,1\}^m$, $\Gamma\subseteq[n]$, and $|\Gamma|=m$, let $\bx=(A_\Gamma^\top)^{-1}\bb$. For any $i\in[m]$, $|\bx_i|\geq\gamma$.
	\end{enumerate}
	Define $\gamma(A)$ to be the largest $\gamma$ such that $A$ is $\gamma$-nice. We say $A$ is \textit{nice} if $\gamma(A)>0$.
\end{definition}

To motivate the definition of niceness, the following lemma shows that the $\ell_1$ minimization problem defined by matrix $A$ has unique solution if $\gamma(A)>0$.
\begin{lemma}
	Let $A\in\R^{m\times n}$. If $\gamma(A)>0$, then for any $\bb\in\R^m$, the $\ell_1$ minimization problem for $(A,\bb)$ has unique solution.
\end{lemma}
\begin{proof}
	We prove the lemma by contradiction. Suppose there exists $\bb\in\R^m$ such that there are two distinct solutions $\bx_1\neq\bx_2$ to the $\ell_1$ minimization problem for $(A,\bb)$. Let $\bv^*$ be the optimal solution of the dual program as in equation~\eqref{op:basispursuit-dual}. By the complementary slackness in the KKT condition, for any optimal solution $\bx$ to the primal program, $\text{supp}(\bx)\subseteq\{i\in[n]:\ |A_i^\top\bv^*|=1 \}$. Let $S=\{i\in[n]:\ |A_i^\top\bv^*|=1 \}$, then both $\bx_1$ and $\bx_2$ are solution to $A_S\bx=\bb_S$ where $A_S$ and $\bb_S$ are restrictions to index set $S$.  As $\gamma(A)>0$, we have $|S|\leq m$. By the non-degeneracy of $A$, $A_S$ has full rank and thus $A_S\bx=\bb_S$ has unique solution. That is, $\bx_1=\bx_2$, which is a contradiction.
	
	We conclude that if $A$ is non-degenerate and $\gamma(A)>0$, then for any $\bb\in\R^m$, the $\ell_1$ minimization problem for $(A,\bb)$ has unique solution.
\end{proof}

In general, it is easy to find a matrix $A$ such that $\gamma(A)=0$. However, we would like to argue that most of the matrices are actually nice.
The following lemma shows that random matrix $A$ sampled from the \textit{rotational symmetry model (RSM)} is nice. In RSM, each column of $A$ is an uniform vector on the unit sphere of $\R^m$. Note that such matrix for $\ell_1$ minimization problem is commonly used in practice such as compressed sensing.
\begin{lemma}\label{lem:gamma lb of RSM}
	Let $A\in\R^{m\times n}$ be a random matrix samples from RSM, then $\gamma(A)>0$ with high probability.
\end{lemma}
\begin{proof}
	First, we show that $A$ is non-degenerate with high probability. For any $\Gamma\subseteq[n]$ and $i\in\Gamma$, denote the event where $A_i = \Pi_{A_{\Gamma\backslash\{i\}}}A_i$ as $E_{\Gamma,i}$. Note that this event is measured zero for all choice of $\Gamma$ and $i$ and thus by union bound, we have $A$ being non-degenerate with high probability. For the other two properties, similar arguments hold.
\end{proof}

We remark that giving a lower bound in terms of $m$ and $n$ for $\gamma(A)$ would result in a better asymptotic bound for our main theorem and could have applications in other problems too. Since the goal of this paper is giving a provable analysis, we do not intend to optimize the parameter. Note that for $A$ sampled from RSM, $\gamma(A)$ has an inverse exponential lower bound directly from union bound when $n$ and $m$ are polynomially related. As for upper bound, there are inverse quasi-polynomial upper bound if $n\geq\polylog(m)\cdot m$ and inverse exponential upper bound if $n\geq m^{1+\Omega(1)}$ as pointed out by the anonymous reviewer from ITCS 2019. See~\autoref{sec:quasi poly ub for gamma of RSM}. for more details.
We leave it as an open question to understand the correct asymptotic behavior of $\gamma(A)$ when $A$ is sampled from RSM.

\subsection{Ideal coupling}\label{sec:ideal coupling}
Ideal coupling is a technique to keeping track of the dual SNN $\bv(t)$ by associating any point in the dual polytope to a point in a smaller polytope. Concretely, let $\mathcal{P}_{A,1}=\{\bv\in\R^m:\ \|A^\top\bv\|_\infty\leq1\}$ be the dual polytope and $\mathcal{P}_{A,1-\tau}$ be the \textit{ideal polytope} where $\tau\in(0,1)$ is  an important parameter that will be properly chosen\footnote{The choice of $\tau$ depends on $A$ and $1$ and will be discussed later.} in the end of the proof. Observe that $\mathcal{P}_{A,1-\tau}\subsetneq\mathcal{P}_{A,1}$. The idea of ideal coupling is associating each $\bv\in\mathcal{P}_{A,1}$ with a point $\bv^{\text{ideal}}$ in $\mathcal{P}_{A,1-\tau}$. In the analysis, we will then focus on the dynamics of $\bv^\text{idael}$ instead of that of $\bv$.

Before we formally define the coupling, we have to define a \textit{partition} of $\mathcal{P}_{A,1}$ with respect to $\mathcal{P}_{A,1-\tau}$ as follows.

\begin{definition}[partition of $\mathcal{P}_{A,1}$]\label{def:ideal partition}
	Let $\mathcal{P}_{A,1}$ and $\mathcal{P}_{A,1-\tau}$ be defined as above. For each $\bv^\text{ideal}\in\mathcal{P}_{A,1-\tau}$, define
	\begin{equation*}
	S_{\bv^\text{ideal}} = \{\bv^\text{ideal}+\mathcal{C}_{A,\Gamma(\bv^{ideal})}\}\cap\mathcal{P}_{A,1}.
	\end{equation*}
	where $\Gamma(\bv^\text{ideal})=\{i\in[\pm n]:\ \langle A_i,\bv^\text{ideal}\rangle=1-\tau\}$ is the active walls of $\bv^\text{ideal}$ and $\mathcal{C}_{A,\Gamma(\bv^{ideal})}=\{\sum_{i\in\Gamma(\bv^\text{ideal})}a_iA_i,\ \forall a_i\geq0\}$ is the cone spanned by the column of $A$ indexed by $\Gamma(\bv^\text{ideal})$.
\end{definition}

\begin{exmp}
	Consider the example where $A=\bigl( \begin{smallmatrix}1&0\\0&1\end{smallmatrix}\bigr)$ and $\tau\in(0,1)$. The dual polytope (resp. ideal polytope) is the square with vertices in the form $(\pm1,\pm1)$ (resp. $(\pm1-\tau,\pm1-\tau)$). For a arbitrary $\bv^\text{ideal}=(x,y)\in\mathcal{P}_{A,1-\tau}$, let us see what $S_{\bv^\text{ideal}}$ is:
	\begin{itemize}
		\item When $|x|,|y|<1-\tau$, \textit{i.e.,} $\bv^\text{ideal}$ strictly lies inside $\mathcal{P}_{A,1-\tau}$, $\Gamma(\bv^\text{ideal})=\emptyset$ and thus $C_{A,\Gamma(\bv^\text{ideal})}=\emptyset$. Namely, $S_{\bv^\text{ideal}}=\bv^\text{ideal}$.
		\item When $|x|=1-\tau$ and $|y|<1-\tau$, \textit{i.e.,} $\bv^\text{ideal}$ lies on an edge of the ideal polytope, $\Gamma(\bv^\text{ideal})=\{\text{sgn}(x)\cdot1\}$ and thus $C_{A,\Gamma(\bv^\text{ideal})}=\{(a,0):\ a\geq0\}$. Namely, $S_{\bv^\text{ideal}}=\{(a,y):\ a\in[1-\tau,1] \}$.
		\item When $|x|<1-\tau$ and $|y|=1-\tau$, \textit{i.e.,} $\bv^\text{ideal}$ lies on an edge of the ideal polytope, $\Gamma(\bv^\text{ideal})=\{\text{sgn}(y)\cdot2\}$ and thus $C_{A,\Gamma(\bv^\text{ideal})}=\{(0,b):\ b\geq0\}$. Namely, $S_{\bv^\text{ideal}}=\{(x,b):\ b\in[1-\tau,1] \}$.
		\item When $|x|=|y|=1-\tau$, \textit{i.e.,} $\bv^\text{ideal}$ lies on a vertex of the ideal polytope, $\Gamma(\bv^\text{ideal})=\{\text{sgn}(x)\cdot1,\text{sgn}(y)\cdot2\}$ and thus $C_{A,\Gamma(\bv^\text{ideal})}=\{(a,b):\ a,b\geq0\}$. Namely, $S_{\bv^\text{ideal}}=\{(a,b):\ a,b\in[1-\tau,1] \}$.
	\end{itemize}
\end{exmp}

The following lemma checks that Definition~\ref{def:ideal partition} does give a partition for $\mathcal{P}_{A,1}$.
\begin{lemma}\label{lem:ideal partition}
	$\{S_{\bv^\text{ideal}}\}_{\bv^\text{ideal}\in\mathcal{P}_{A,1-\tau}}$ is a partition for $\mathcal{P}_{A,1}$.
\end{lemma}
\begin{proof}[Proof of Lemma~\ref{lem:ideal partition}]
	The proof is basically doing case analysis and using some basic properties from linear algebra. See Section~\ref{sec:missing proofs ideal auxiliary SNN} for details.
\end{proof}

\begin{definition}[ideal coupling]\label{def:ideal coupling}
	Let $\mathcal{P}_{A,1}$ and $\mathcal{P}_{A,1-\tau}$ be defined as above. For any $\bv\in\mathcal{P}_{A,1}$, define $\bv^\text{ideal}(\bv)$ be the unique $\bv^\text{ideal}$ such that $\bv\in S_{\bv^\text{ideal}}$. We denote $\bv^\text{ideal}(\bv)$ as $\bv^\text{ideal}$ when the context is clear. Specifically, for any $t\geq0$, we denote $\bv^\text{ideal}(t)=\bv^\text{ideal}(\bv(t))$ as the ideal SNN at time $t$.
\end{definition}

See Figure~\ref{fig:ideal coupling} for an example of the ideal coupling.

\begin{figure}[h]
	\centering
	\includegraphics[width=10cm]{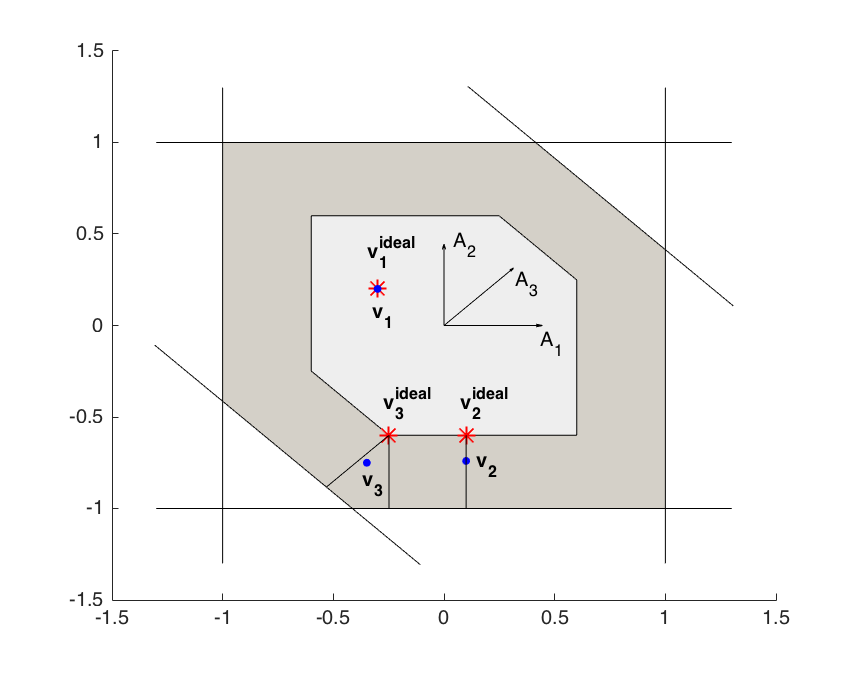}
	\caption{This is an example of ideal coupling in $\R^2$ where $\tau=0.4$, $A_1=[1\ 0]^\top$, $A_2=[0\ 1]^\top$, and $A_3=[\frac{1}{\sqrt{2}}\ \frac{1}{\sqrt{2}}]^\top$. The dots (\textit{i.e.,} $\bv_1,\bv_2,\bv_3$) are dual SNN and the stars (\textit{i.e.,} $\bv_1^\text{ideal},\bv_2^\text{ideal},\bv_3^\text{ideal}$) are the corresponding ideal SNN. The whole gray area is the dual polytope $\mathcal{P}_{A,1}$ and the gray area in the middle is the ideal polytope $\mathcal{P}_{1-\tau}$.}
	\label{fig:ideal coupling}
\end{figure}
\vspace{3mm}

Note that Definition~\ref{def:ideal coupling} is well-defined due to Lemma~\ref{lem:ideal partition}. With the ideal coupling, we are then switching to analyze the \textit{ideal SNN} $\bv^\text{ideal}(t)$ instead of the dual SNN $\bv(t)$. In the following, we are going to show that the ideal SNN is indeed tractable for analysis, though it is highly non-trivial and is very sensitive to the choice of parameters.

To show the convergence of ideal SNN, we need a notion to measure how close $\bv^\text{ideal}(t)$ and the optimal point is. To do so, we define the \textit{ideal solution} of ideal SNN at time $t$ as follows.

\begin{definition}[ideal solution]\label{def:ideal solution}
	For any $t\geq0$, define the ideal solution $\bx^\text{ideal}(t)$ at time $t$ as
	\begin{equation*}
	\bx^\text{ideal}(t) = \argmin_{\substack{\bx\geq0,\\\bx_i=0,\ \forall i\in\Gamma(\bv^\text{ideal}(t))}}\|\bb-A\bx\|_2.
	\end{equation*}
	Also, let $\Gamma^*(\bv^\text{ideal}(t))=\{i\in\Gamma(\bv^\text{ideal}(t)):\ \bx^\text{ideal}(t)\neq0\}$ to be the set of super active neurons.
\end{definition}

In the later proof, we need one more definition on a variant of ideal SNN called the \text{super SNN}. Similar to Definition~\ref{def:ideal solution}, we define the super ideal SNN $\bv^\text{super}(t)$ as the projection of $\bv(t)$ to the ideal polytope \textit{without} those non-super ideal neurons. Formally, define $\bv^\text{super}(t)$ be the unique solution of the following equations: $\bv=\bv(t)-A_{\Gamma^*(\bv^\text{ideal}(t))}\bz$ and $A_i^\top\bv=1-\tau$ for each $i\in\Gamma^*(\bv^\text{ideal}(t))$. See Figure~\ref{fig:super} for example. Note that the uniqueness of the solution is guaranteed by the non-degeneracy of $A$.

\begin{figure}[h]
	\centering
	\includegraphics[width=10cm]{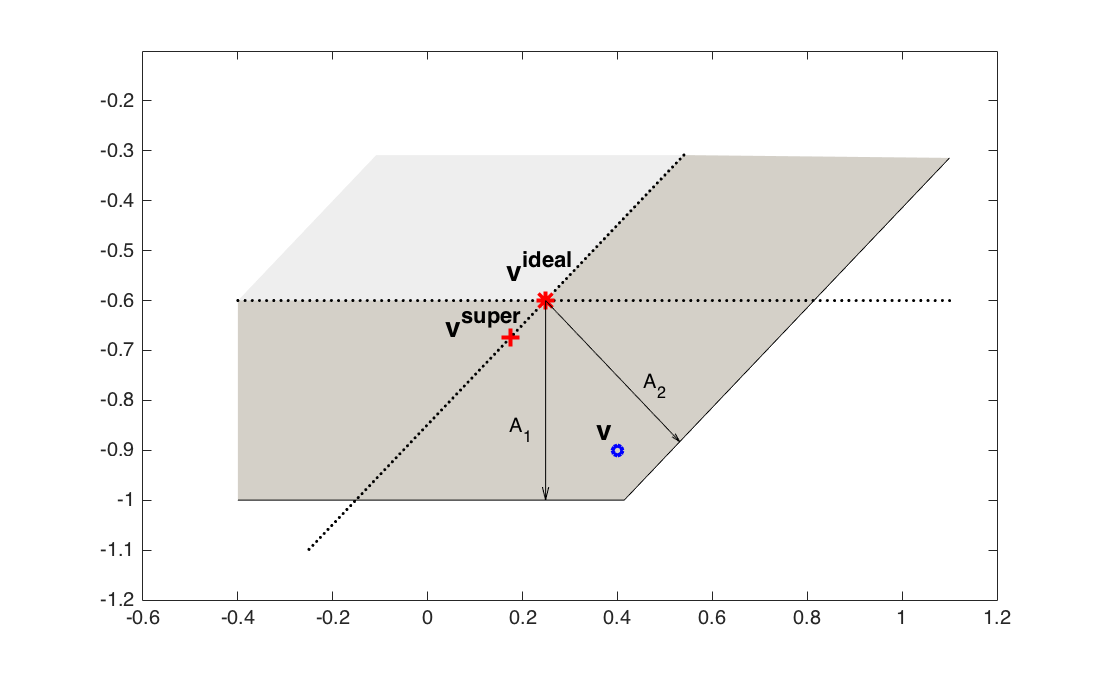}
	\caption{This is an example of $\bv^\text{super}$ in $\R^2$ where $\tau=0.4$, $A_1=[0\ -1]^\top$, $A_2=[\frac{1}{\sqrt{2}}\ -\frac{1}{\sqrt{2}}]^\top$, $\bb=[1\ 0]^\top$, and $\bv=[0.4\ -0.9]^\top$. The light gray area is the ideal polytope and the dark gray area is the dual polytope. In this example, we have $\Gamma(\bv)=\{1,2\}$ while $\Gamma^*(\bv)=\{2\}$. As a result, $\bv^\text{super}$ is defined as the projection of $\bv$ onto the ideal polytope that only contains neuron $2$.}
	\label{fig:super}
\end{figure}
\vspace{3mm}

It is indeed unclear why we need these definitions at this stage of the proof. It would be clearer why we need them in the next two subsections once we go into the main analysis. Before we move on to more details, see Figure~\ref{fig:ideal coupling} and Figure~\ref{fig:super} again to familiarize with the definitions.

\subsection{Ideal SNN remains unchanged after firing spikes}\label{sec:unchange}
In this subsection, we are going to prove the following important lemma saying that the dual SNN would not change its ideal SNN after firing spikes.

\begin{lemma}[ideal SNN remains unchanged after firing spikes]\label{lem:ideal SNN unchaged}
	There exists a polynomial $\alpha(\cdot)$ such that if $A$ is nice and $0<\alpha\leq\alpha(\frac{\tau\cdot\gamma(A)}{n\cdot\lambda_{\max}})$, then $\bv(t)-\alpha A\bs(t)\in S_{\bv^\text{ideal}(t)}$ for each $t\geq0$.
\end{lemma}

\begin{proof}[Proof of Lemma~\ref{lem:ideal SNN unchaged}]
	First, note that for each $\bv\in S_{\bv^\text{ideal}(t)}$, by the property of dual polytope, there exists an unique $\bz\in\R_{\geq0}^{|\Gamma(\bv^\text{ideal}(t))|}$ such that $\bv=\bv^\text{ideal}(t)+A_{\Gamma(\bv^\text{ideal}(t))}\bz$ where $\bz$ can be thought of as the \textit{coordinates} of $\bv$ in $S_{\bv^\text{ideal}(t)}$. With this concept in mind, it is then sufficient to show that whenever neuron $i$ fires, $\bz_i>\alpha$. The reason is that
	\begin{align}
	\bv(t)-\alpha A\bs(t) &= \bv^\text{ideal}(t) + A_{\Gamma(\bv^\text{ideal}(t))}\bz - \sum_{i\in\Gamma(\bv(t))}\alpha A_i\nonumber\\
	&= \bv^\text{ideal}(t) + \sum_{i\in\Gamma(\bv^\text{ideal}(t))\backslash\Gamma(\bv(t))}\bz_iA_i + \sum_{i\in\Gamma(\bv(t))}(\bz_i-\alpha)A_i.\label{eq:ideal SNN fire}
	\end{align}
	As a result, if $\bz_i-\alpha>0$ for every $i\in\Gamma(\bv(t))$, then we have $\bv(t)-\alpha A\bs(t)\in S_{\bv^\text{ideal}(t)}$ because every new coordinates are still non-negative. See Figure~\ref{fig:coordinate} for an example.
	
	\begin{figure}[h]
		\centering
		\includegraphics[width=10cm]{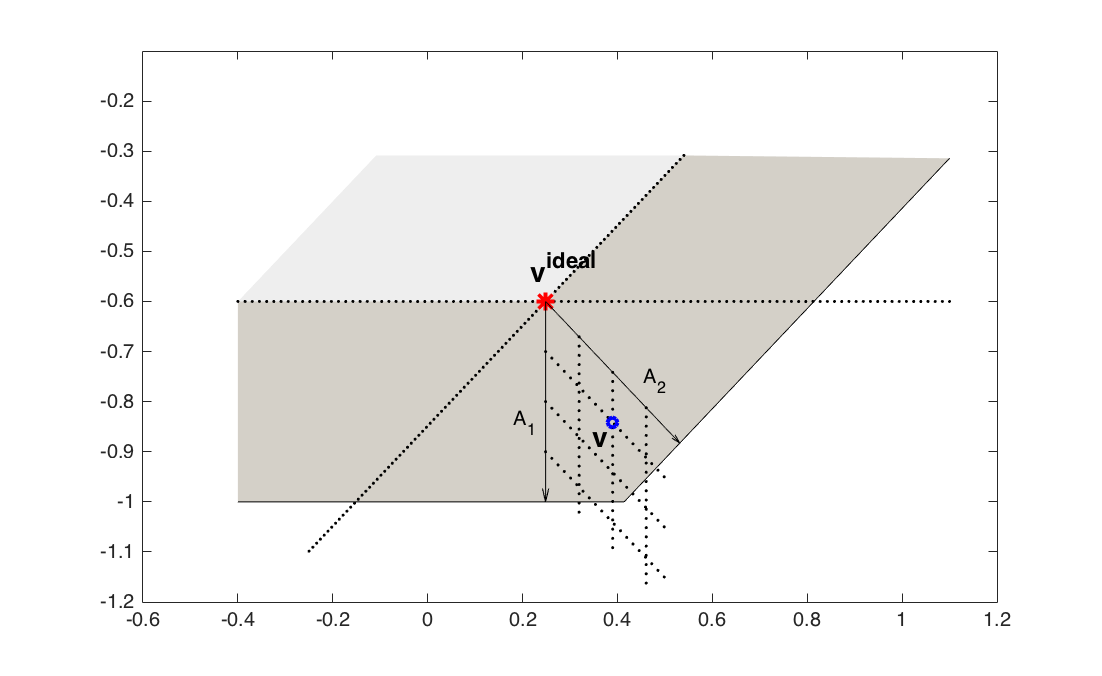}
		\caption{This is an example of \textit{coordinates} of  $S_{\bv^\text{ideal}}$ in $\R^2$ where $\tau=0.4$, $A_1=[0\ -1]^\top$ and $A_2=[\frac{1}{\sqrt{2}}\ -\frac{1}{\sqrt{2}}]^\top$. The light gray area is the ideal polytope and the dark gray area is the dual polytope. In this example, the dot lines are the \textit{level set} of each coordinates in $S_{\bv\text{ideal}}$. For instance, the $\bv$ in the figure has coordinate $\bz=[0.1\ 0.2]^\top$ and thus we have $\bv=\bv^\text{ideal}+A\bz$.}
		\label{fig:coordinate}
	\end{figure}
	\vspace{3mm}
	
	\begin{claim}\label{claim:ideal SNN stay after fire coordinate}
		There exists a polynomial $\alpha(\cdot)$ such that when $0<\alpha\leq\poly(\frac{\tau\cdot\gamma(A)}{n\cdot\lambda_{\max}})$ and $\bv(t)=\bv^\text{ideal}(t)+A_{\Gamma(\bv^\text{ideal}(t))}\bz\in S_{\bv^\text{ideal}(t)}$ for some $t\geq0$, if $i\in\Gamma(\bv(t))$, then $\bz_i>\alpha$.
	\end{claim}
	\begin{proof}[Proof of Claim~\ref{claim:ideal SNN stay after fire coordinate}]
		The proof consists of two steps. First, we are going to show that for any $t\geq0$, $\bv(t)$ is close to $\bv^\text{ideal}(t)$. Concretely, if $\alpha\leq\frac{\tau}{m}$, then $\|\bv(t)-\bv^\text{ideal}(t)\|_2\leq\tau\lambda_{\max}$. Second, we are going to show that once we pick $\alpha$ small enough, then for any $i\in\Gamma(\bv^\text{ideal}(t))$, the wall $W_i$ is far away from the $\alpha$-level set in $S_{\bv^\text{ideal}(t)}$. Thus, whenever neuron $i$ fires, $\bz_i>\alpha$.
		
		The first step is a key observation that the distance between $\bv(t)$ and $\bv^\text{ideal}(t)$ would not increase after the neurons fire spikes. The main reason is that neuron $i$ fires at time $t$ if and only if $A_i^\top\bv(t)>1$. As a result,
		\begin{align*}
		\|\left(\bv(t)-\alpha A\bs(t)\right)-\bv^\text{ideal}(t)\|_2^2 &= \|\bv(t)-\bv^\text{ideal}(t)\|_2^2 + \alpha^2\|A\bs(t)\|_2^2 - 2\alpha\left(A\bs(t)\right)^\top\left(\bv(t)-\bv^\text{ideal}(t)\right)\\
		& = \|\bv(t)-\bv^\text{ideal}(t)\|_2^2 + \alpha^2\|A\bs(t)\|_2^2 -2\alpha\sum_{i\in\Gamma(\bv(t))}A_i^\top\left(\bv(t)-\bv^\text{ideal}(t)\right)\\
		&\leq\|\bv(t)-\bv^\text{ideal}(t)\|_2^2 + \alpha^2|\Gamma(\bv(t))|^2 -2\alpha\tau|\Gamma(\bv(t))|.
		\end{align*}
		That is, if $\alpha\leq\frac{\tau}{m}$, then $\|\left(\bv(t)-\alpha A\bs(t)\right)-\bv^\text{ideal}(t)\|_2$ would not increase after some neurons fire. Furthermore, the longest distance between $\bv(t)$ and $\bv^\text{ideal}(t)$ would then be $\tau\lambda_{\max}$.
		
		The second step is rather complicated. Let us start with some definitions. Recall that for any $i\in[\pm n]$, the wall $i$ is defined as $W_i=\{\bv\in\R^m:\ A_i^\top\bv=1 \}$. Now, define the $\alpha$-level set of $i$ in $S_{\bv^\text{ideal}(t)}$ as
		$$
		L_{\bv^\text{ideal}(t),i,\alpha} = \{\bv\in S_{\bv^\text{ideal}(t)}:\ \bv=\bv^\text{ideal}(t)+A_{\Gamma(\bv^\text{ideal}(t))}\bz,\ \bz_i=\alpha\}. 
		$$
		That is, $L_{\bv^\text{ideal}(t),i,\alpha} $ consists of the set of points in $S_{\bv^\text{ideal}(t)}$ that has the $i^\text{th}$ coordinate to be $\alpha$.
		
		\begin{claim}[furtherest point in $S_{\bv^\text{ideal}}$]\label{claim:furthest point in safe region}
			For any $t\geq0$, let $\bv_{\Gamma(\bv^\text{ideal}(t))}$ be the unique point $\bv\in S_{\bv^\text{ideal}(t)}$ such that for any $i\in\Gamma(\bv^\text{ideal}(t))$, $A_i^\top\left(\bv-\bv^\text{ideal}(t)\right)=\tau$. Then, we have $\|\bv_{\Gamma(\bv^\text{ideal}(t))}-\bv^\text{ideal}(t)\|_2=\max_{\bv\in S_{\bv^\text{ideal}(t)}}\|\bv-\bv^\text{ideal}(t)\|_2$.
		\end{claim}
		\begin{proof}[Proof of Claim~\ref{claim:furthest point in safe region}]
			Let us prove by contradiction. Suppose $\bv^*\in S_{\bv^\text{ideal}(t)}$ such that $\|\bv_{\Gamma(\bv^\text{ideal}(t))}-\bv^\text{ideal}(t)\|_2<\|\bv^*-\bv^\text{ideal}(t)\|_2$. To simplify the notations, let $\bv_{\Gamma}=\bv_{\Gamma(\bv^\text{ideal}(t))}-\bv^\text{ideal}(t)$ and $\bv=\bv^*-\bv^\text{ideal}(t)$.
			
			By definition, we have $A_i^\top\bv_\Gamma=\tau$ for all $i\in\Gamma(\bv^\text{ideal}(t))$ and $\bv=A_{\Gamma(\bv^\text{ideal}(t))}\bz_\Gamma$ for some $\bz_\Gamma\in\R_{>0}$. On the other hand, we also have $0\leq A_i^\top\bv\leq\tau$ for all $i\in\Gamma(\bv^\text{ideal}(t)$.
			
			Now, look at the quantity $\bv_{\Gamma}^\top\left(\bv-\bv_\Gamma\right)$. Note that since $\|\bv\|_2>\|\bv_\Gamma\|_2$, we have $\bv_{\Gamma}^\top\left(\bv-\bv_\Gamma\right)>0$. Also, for any $i\in\Gamma(\bv^\text{ideal}(t))$, we have $A_i^\top\left(\bv-\bv_\Gamma\right)\leq0$. Using the fact that $\bv=A_{\Gamma(\bv^\text{ideal}(t))}\bz_\Gamma$ for some $\bz_\Gamma\in\R_{>0}$, we have
			\begin{align*}
			0 < \bv_{\Gamma}^\top\left(\bv-\bv_\Gamma\right)&=\bz_\Gamma^\top A_{\Gamma(\bv^\text{ideal}(t))}^\top\left(\bv-\bv_\Gamma\right)\\
			&= \bz_\Gamma^\top\bu\leq0,
			\end{align*}
			where $\bu=A_{\Gamma(\bv^\text{ideal}(t))}^\top\left(\bv-\bv_\Gamma\right)\in\R_{\leq0}^{|\Gamma(\bv^\text{ideal}(t))|}$. That is, we reach a contradiction and thus $\|\bv\|_2\leq\|\bv_\Gamma\|_2$ and we conclude that $\bv_{\Gamma(\bv^\text{ideal}(t))}$ is the furtherest point from $\bv^\text{ideal}(t)$ in $S_{\bv^\text{ideal}(t)}$.
		\end{proof}
		\begin{claim}[intersection of wall and $\alpha$-level set is far]\label{claim:ideal SNN intersection is far}
			When $0<\alpha\leq\tau^2\cdot\gamma(A)^3$, for any $t\geq0$ and $i\in\Gamma(\bv^\text{ideal}(t))$, we have
			$$
			\min_{\bv:\ \bv\in W_i\cap L_{\bv^\text{ideal}(t),i,\alpha}}\|\bv-\bv^\text{ideal}(t)\|_2>\|\bv_{\Gamma(\bv^\text{ideal}(t))}-\bv^\text{ideal}(t)\|_2.
			$$
		\end{claim}
		\begin{proof}[Proof of Claim~\ref{claim:ideal SNN intersection is far}]
			First, let us write $\bv_{\Gamma(\bv^\text{ideal}(t))}=\bv^\text{ideal}(t)+\sum_{i\in\Gamma(\bv^\text{ideal}(t))}\bz_iA_i$ where $\bz_i\geq\tau\cdot\gamma(A)$ by Definition~\ref{def:nice}. Furthermore, for any $i\in\Gamma(\bv^\text{ideal}(t))$, we have
			\begin{equation*}
			\text{dist}\left(\bv_{\Gamma(\bv^\text{ideal}(t))},\text{span}(A_{\Gamma(\bv^\text{ideal}(t))\backslash\{i\}})\right)\geq|\bz_i|\cdot\text{dist}\left(A_i,\text{span}(A_{\Gamma(\bv^\text{ideal}(t))\backslash\{i\}})\right)\geq\tau\cdot\gamma(A)^2,
			\end{equation*}
			where the last inequality follows Definition~\ref{def:nice}. Namely, if we pick $0<\alpha<\tau^2\cdot\gamma(A)^3$, then
			\begin{equation*}
			\text{dist}\left(\bv_{\Gamma(\bv^\text{ideal}(t))},L_{\bv^\text{ideal}(t),i,\alpha}\right)>0
			\end{equation*}
			and $\bv_{\Gamma(\bv^\text{ideal}(t))}\in\text{Cone}(A_i,L_{\bv^\text{ideal}(t),i,\alpha})$ because $\bz_i\geq\gamma(A)$. Finally, observe that for any $\bv\in W_i\cap L_{\bv^\text{ideal}(t),i,\alpha}$, we have $\bv_{\Gamma(\bv^\text{ideal}(t))}^\top\left(\bv-\bv_{\Gamma(\bv^\text{ideal}(t))}\right)>0$. This completes the proof of Claim~\ref{claim:ideal SNN intersection is far}.
		\end{proof}
		Combine Claim~\ref{claim:furthest point in safe region} and Claim~\ref{claim:ideal SNN intersection is far}, we know that when neuron $i$ fires, the corresponding coordinate $\bz_i$ will be at least $\alpha$. This completes the proof of Claim~\ref{claim:ideal SNN stay after fire coordinate}.
	\end{proof}
	Now, Lemma~\ref{lem:ideal SNN unchaged} follows from Claim~\ref{claim:ideal SNN stay after fire coordinate} and equation~\eqref{eq:ideal SNN fire}.
\end{proof}

\subsection{Strict convergence of ideal SNN and auxiliary SNNs}\label{sec:strict}
In this subsection, the goal is to characterize the dynamics of both ideal and auxiliary SNN. Before defining auxiliary SNN, let us first see the following lemma about the dynamics of ideal SNN. 

\begin{lemma}[dynamics of ideal SNN]\label{lem:ideal SNN dynamics}
	If $A$ is nice, then for any $t\geq0$, we have
	$$\bv^\text{ideal}(t+dt)=\bv^\text{ideal}(t) + \left(\bb-\Pi_{A_{\Gamma(\bv^\text{ideal}(t))}}\bb\right)dt.$$
\end{lemma}
\begin{proof}[Proof of Lemma~\ref{lem:ideal SNN dynamics}]
	We consider two cases: (i) there is no neuron fires any spike and (ii) there is a neuron fires a spike.
	
	\textbf{Case (i)}: By Definition~\ref{def:ideal coupling}, $\bv(t)=\bv^\text{ideal}+A_{\Gamma(\bv^\text{ideal}(t))}\bz$ for some $\bz\geq0$. Also, rewrite the updates $\bb$ as 
	$$\bb=\left(\bb-\Pi_{A_{\Gamma(\bv^\text{ideal}(t))}}\bb\right)+\Pi_{A_{\Gamma(\bv^\text{ideal}(t))}}\bb.$$
	First, $A_i^\top\left(\bb-\Pi_{A_{\Gamma(\bv^\text{ideal}(t))}}\bb\right)=0$ for each $i\in\Gamma(\bv^\text{ideal}(t))$. Next, since there is no neuron fires at time $t$, observe that $\bv(t)+\Pi_{A_{\Gamma(\bv^\text{ideal}(t))}}\bb\in S_{\bv^\text{ideal}(t)}$. Finally, since $\bb-\Pi_{A_{\Gamma(\bv^\text{ideal}(t))}}\bb$ is orthogonal to the subspace spanned by the active neurons, we then have $\bv(t)+\bb dt\in S_{\bv^\text{ideal}(t)+(\bb-\Pi_{A_{\Gamma(\bv^\text{ideal}(t))}}\bb)dt}$.
	
	\textbf{Case (ii)}: To handle spikes, the idea is to focus on the spike term first, and once $\bv(t)$ goes back to the interior of the dual polytope, then it becomes case (i). Here, we use an assumption that if there are some neurons fire at time $t$ and they trigger consecutive firing, we add the external charging \textit{after} the consecutive firing. As a result, it suffices to show that $\bv(t)-\alpha A\bs(t)\in S_{\bv^\text{ideal}(t)}$, which immediately follows from Lemma~\ref{lem:ideal SNN unchaged}.
	
	We conclude that for any $t\geq0$, $\bv^\text{ideal}(t+dt)=\bv^\text{ideal}(t) + \left(\bb-\Pi_{A_{\Gamma(\bv^\text{ideal}(t))}}\bb\right)dt$.
\end{proof}

From Lemma~\ref{lem:ideal SNN dynamics}, one can see that the improvement of ideal SNN is not proportional to the residual error when the $\Pi_{A_{\Gamma(\bv^\text{ideal}(t))}}\neq A\bx^\text{ideal}(t)$. As a result, we have to design a bunch of \textit{auxiliary SNN} to make sure that at least one of them has improvement proportional to the residual error. The auxiliary SNNs $\{\bv^\text{auxiliary}_d(t)\}_{d\in[m-1]}$ is defined as follows.
\begin{definition}[auxiliary SNNs]\label{def:auxiliary}
	For each $t\geq0$, and $d\in[m-1]$, define $\bv^\text{auxiliary}(0)=\mathbf{0}$ and
	\[
	\bv^\text{auxiliary}_d(t+dt) = \left\{
	\begin{array}{ll}
	\bv^\text{auxiliary}_d(t)+\left(\bb-A\bx^\text{ideal}(t)\right)dt&\text{, if }|\Gamma^*(\bv^\text{ideal}(t+dt))|=d\\&\text{ and } |\Gamma^*(\bv^\text{ideal}(t))|=d,\\
	\bv^\text{super}(t+dt)&\text{, if }|\Gamma^*(\bv^\text{ideal}(t+dt))|=d\\&\text{ and }|\Gamma^*(\bv^\text{ideal}(t))|\neq d,\\
	\bv_d^\text{auxiliary}(t)&\text{, else}.
	\end{array}
	\right.
	\]
\end{definition}

The auxiliary SNNs have the following important property that is crucial in the proof of the Lemma~\ref{lem:strict improvement} which gives the strict improvement guarantee.

\begin{lemma}[auxiliary SNNs jump]\label{lem:auxiliary SNN jump}
	Suppose $A$ is nice and $\tau=O(\frac{\gamma(A)}{n^2\cdot\lambda_{\max}^2})$. For any $t>0$ and $d\in[m-1]$, if $|\Gamma^*(\by^\text{ideal}(t))|\neq|\Gamma^*(\bv^\text{ideal}(t+dt))|=d$, then $\bb^\top\left(\bv^\text{auxiliary}_d(t+dt)-\bv^\text{auxiliary}_d(t)\right)>0$.
\end{lemma}
\begin{proof}[Proof of Lemma~\ref{lem:auxiliary SNN jump}]
	By the definition of auxiliary SNNs, we have three observations. First, $\|\bv^\text{auxiliary}_d(t+dt)-\bv^\text{ideal}(t)\|_2=\|\bv^\text{super}(t+dt)-\bv^\text{ideal}(t)\|_2=O(\tau\cdot n\cdot\lambda_{\max})$. Second, there exists $0\leq t'<t$ such that $\bv^\text{auxiliary}_d(t)=\bv^\text{super}(t')$ and $\Gamma(\bv^\text{ideal}(t'))\neq\Gamma(\bv^\text{ideal}(t))$. That is, we also have $\|\bv^\text{auxiliary}_d(t)-\bv^\text{ideal}(t')\|_2=\|\bv^\text{super}(t')-\bv^\text{ideal}(t)\|_2=O(\tau\cdot n\cdot\lambda_{\max})$. Finally, since $\Gamma(\bv^\text{ideal}(t'))\neq\Gamma(\bv^\text{ideal}(t))$, by Lemma~\ref{lem:ideal SNN dynamics}, we have $\bb^\top\left(\bv^\text{ideal}(t)-\bv^\text{ideal}(t')\right)=\Omega(\|\bb\|_2\cdot\frac{\gamma(A)}{n\cdot\lambda_{\max}})$. Combine the three we have 
	\begin{align*}
	\bb^\top\left(\bv^\text{auxiliary}_d(t+dt)-\bv^\text{auxiliary}_d(t)\right)&\geq\bb^\top\left(\bv^\text{ideal}(t)-\bv^\text{ideal}(t')\right) - O(\|\bb\|_2\cdot\tau\cdot n\cdot\lambda_{\max})\\
	&\geq\Omega(\|\bb\|_2\cdot\frac{\lambda(A)}{n\cdot\lambda_{\max}}) - O(\|\bb\|_2\cdot\tau\cdot n\cdot\lambda_{\max})>0,
	\end{align*}
	where the last inequality holds when we pick $\tau=O(\frac{\gamma(A)}{n^2\lambda_{\max}^2})$.
\end{proof}

Now, we are able to prove the main lemma about identifying a potential function that is strictly improving as long as $\bx^\text{ideal}(t)$ is not the optimal solution for $\ell_1$ minimization problem.

\begin{lemma}[strict improvement]\label{lem:strict improvement}
	For any $t>0$, we have
	\begin{equation*}
	\frac{d}{dt}\bb^\top\left(\bv^\text{ideal}(t)+\sum_{d\in[m-1]}\bv^\text{auxiliary}_d(t)\right)\geq\bb^\top A\bx^\text{ideal}(t).
	\end{equation*}
\end{lemma}
\begin{proof}[Proof of Lemma~\ref{lem:strict improvement}]
	The proof is based on case analysis on the size of $|\Gamma^*(\bv^\text{ideal}(t))|$. We consider three cases:
	\begin{enumerate}[label=(\roman*)]
		\item $\Gamma^*(\bv^\text{ideal}(t))=\Gamma(\bv^\text{ideal}(t))$, 
		\item $\Gamma^*(\bv^\text{ideal}(t))\subsetneq\Gamma(\bv^\text{ideal}(t))$ and $|\Gamma^*(\bv^\text{ideal}(t))|=|\Gamma^*(\bv^\text{ideal}(t+dt))|$, and 
		\item $\Gamma^*(\bv^\text{ideal}(t))\subsetneq\Gamma(\bv^\text{ideal}(t))$ and $|\Gamma^*(\bv^\text{ideal}(t))|\neq|\Gamma^*(\bv^\text{ideal}(t+dt))|$.
	\end{enumerate}
	In each case, we are going to show that at least one of $\bv^\text{ideal}(t)$ or $\bv^\text{auxiliary}_d(t)$ for some $d\in[m-1]$ has the desired improvement. Also, we need to show that all of them would not get worse. Formally, we state it as the following claim.
	\begin{claim}\label{claim:ideal auxiliary does not get worse}
		For any $t>0$ and $d\in[m-1]$, $\frac{d}{dt}\bb^\top\bv^\text{ideal}(t),\bb^\top\bv^\text{auxiliary}_d(t)\geq0$.
	\end{claim}
	\begin{proof}[Proof of Claim~\ref{claim:ideal auxiliary does not get worse}]
		From Lemma~\ref{lem:ideal SNN dynamics}, we already have $\bb^\top\bv^\text{ideal}(t)\geq0$. For any $d\in[m-1]$, consider three cases as in Definition~\ref{def:auxiliary}.
		
		If $|\Gamma^*(\bv^\text{ideal}(t))|=|\Gamma^*(\bv^\text{ideal}(t+dt))|=d$, then $\frac{d}{dt}\bb^\top\bv^\text{auxiliary}_d(t)=\bb^\top(A-\bx^\text{ideal}(t))\geq0$.
		
		If $|\Gamma^*(\bv^\text{ideal}(t))|\neq|\Gamma^*(\bv^\text{ideal}(t+dt))|=d$, then by Lemma~\ref{lem:auxiliary SNN jump} we have $\bb^\top\left(\bv^\text{auxiliary}_d(t+dt)-\bv^\text{auxiliary}_d(t)\right)>0$ and thus  $\frac{d}{dt}\bb^\top\bv^\text{auxiliary}_d(t)\geq0$.
		
		Finally, when non of the above happen, we simply have  $\frac{d}{dt}\bb^\top\bv^\text{auxiliary}_d(t)=0$.
	\end{proof}
	With Claim~\ref{claim:ideal auxiliary does not get worse}, it suffices to show that at least one of $\bv^\text{ideal}(t)$ or $\bv^\text{auxiliary}_d(t)$ for some $d\in[m-1]$ has the desired improvement in all of the above three cases.
	
	\textbf{Case (i)}: In this case, $A\bx^\text{ideal}(t)=\Pi_{A_{\Gamma(\bv^\text{ideal}(t))}\bb}$. Thus, by Lemma~\ref{lem:ideal SNN dynamics}, we have $\frac{d}{dt}\bb^\top\bv^\text{ideal}(t)=\bb^\top\left(\bb-A\bx^\text{ideal}(t)\right)$.
	
	\textbf{Case (ii)}: In this case, let $d=|\Gamma^*(\bv^\text{ideal}(t+dt))|$. By Definition~\ref{def:auxiliary}, we have $\frac{d}{dt}\bb^\top\bv^\text{auxiliary}(t)=\bb^\top\left(\bb-A\bx^\text{ideal}(t)\right)$.
	
	\textbf{Case (iii)}: In this case, let $d=|\Gamma^*(\bv^\text{ideal}(t+dt))|$. By Lemma~\ref{lem:auxiliary SNN jump}, we have $\bb^\top\left(\bv^\text{auxiliary}_d(t+dt)-\bv^\text{auxiliary}_d(t)\right)>0$ and thus $\frac{d}{dt}\bb^\top\bv^\text{auxiliary}_d(t)\geq\bb^\top\left(\bb-A\bx^\text{ideal}(t)\right)$.
	
	This completes the proof of Lemma~\ref{lem:strict improvement}.
\end{proof}

Finally, before we go into the final proof for Theorem~\ref{thm:l1}, we need the following lemma about some properties about the ideal solution defined in Definition~\ref{def:ideal solution}.

\begin{lemma}[properties of ideal solution]\label{lem:ideal solution properties}
	For any $t\geq0$, we have the following.
	\begin{enumerate}
		\item $\bb^\top A\bx^{\text{ideal}}(t)=\|A\bx^{\text{ideal}}(t)\|_2^2$,
		\item $\|\bb-A\bx^{\text{ideal}}(t)\|_2^2=\|\bb\|_2^2-\|A\bx^{\text{ideal}}(t)\|_2^2$, and
	\end{enumerate}
\end{lemma}
\begin{proof}[Proof of Lemma~\ref{lem:ideal solution properties}]
	The lemma is directly followed by the following property of conic projection. For any $A\in\R^{m\times n}$, $\bb\in\R^m$, and $\Gamma\subseteq[\pm n]$ be a valid set, we have $\bb^\top A\bx_{A,\bb,\Gamma}=\|A\bx_{A,\bb,\Gamma}\|_2^2$. In the following, we are going to first prove this property of conic projection and then use it to prove the lemma.
	
	Let us rewrite the definition of conic projection as an optimization program.
	\begin{equation}\label{op:conicproj}
	\begin{aligned}
	& \underset{\bx\in\R^n}{\text{minimize}}
	& & \frac{1}{2}\|\bb-A\bx\|_2^2 \\
	& \text{subject to}
	& & \bx_j\geq0,\ j\in\Gamma,\\
	& & & \bx_i=0,\ i,-i\notin\Gamma.
	\end{aligned}
	\end{equation}
	Let $\by$ be the dual variable of~\eqref{op:conicproj} and $\by^*$ be the optimal dual value, the Lagrangian of~\eqref{op:conicproj} is
	\begin{equation*}
	\mathcal{L}(\bx) = \frac{1}{2}\|\bb-A\bx\|_2^2-\by^\top\bx,
	\end{equation*}
	and its gradient is
	\begin{equation*}
	\nabla_{\bx}\mathcal{L}(\bx) = \ATA\bx - A^\top\bb - \by.
	\end{equation*}
	By the KKT condition, we know that the optimal primal solution $\bx_{A,\bb,\Gamma}$ and the optimal dual solution $\by^*$ make the gradient of the Lagrangian diminish.
	\begin{equation}\label{eq:idealalgorithm-conicproj-KKT-gradient}
	\nabla_{\bx}\mathcal{L}(\bx_{A,\bb,\Gamma}) = \ATA\bx_{A,\bb,\Gamma} - A^\top\bb - \by^*=0,
	\end{equation}
	and the complementary slackness
	\begin{equation}\label{eq:idealalgorithm-conicproj-KKT-cs}
	\bx_{A,\bb,\Gamma}^\top\by^*=0.
	\end{equation}
	By~\eqref{eq:idealalgorithm-conicproj-KKT-gradient} and~\eqref{eq:idealalgorithm-conicproj-KKT-cs}, we have
	\begin{equation*}\label{eq:idealalgorithm-conicproj-orthogonal}
	(A\bx_{A,\bb,\Gamma})^\top(\bb-A\bx_{A,\bb,\Gamma})=0.
	\end{equation*}
	As a result, $\bb^\top A\bx_{A,\bb,\Gamma}=\|A\bx_{A,\bb,\Gamma}\|_2^2$.
	
	This completes the proof of Lemma~\ref{lem:ideal solution properties}.
\end{proof}

\subsection{The convergence of dual SNN}\label{sec:l1 convergence}
In this subsection, we are going to prove the main convergence theorem of the dual SNN using ideal and auxiliary SNN. The following lemma says that at least one of ideal SNN or auxiliary SNN improves at each step.

The following lemma shows the monotonicity of the residual error $\|\bb-A\bx^{\text{ideal}}\|_2$.
\begin{lemma}[monotonicity of residual error]\label{lemma:idealalgorithm-primal-nondecreasing}
	There exists a polynomial $\alpha(\cdot)$ such that when $0<\alpha\leq \alpha(\frac{\gamma(A)}{n\cdot\lambda_{\max}})$, we have $\|\bb-A\bx^{\text{ideal}}(t)\|_2$ is non-increasing and $\|A\bx^{\text{ideal}}(t)\|_2$ is non-decreasing in $t$.
\end{lemma}
\begin{proof}[Proof of Lemma~\ref{lemma:idealalgorithm-primal-nondecreasing}]
	Consider two cases.
	\begin{enumerate}[label=(\arabic*)]
		\item When there is a new index joins the active set. Clearly that $\|A\bx^{\text{ideal}}(t)\|_2$ won't decrease since the new cone contains the old one. By Lemma~\ref{lem:ideal solution properties}, we know that $\|\bb-A\bx^{\text{ideal}}(t)\|_2$ is non-increasing.
		\item When there is an index leaves the the active set. Without loss of generality, assume $j\in[\pm n]$ leaves the active set. In the following, we want to show that $\bx^{\text{ideal}}_{|j|}(t)=0$. As the direction of $\bv^{\text{ideal}}(t)$ is $\bb-A\bx^{\text{ideal}}(t)$, it means that $A_j^{\top}(\bb-A\bx^{\text{ideal}}(t))<0$. Suppose $\bx^{\text{ideal}}_{|j|}(t)\neq0$ for contradiction. Since $j$ was in the active set, it is the case that $\bx^{\text{ideal}}_j(t)>0$. Take $0<\epsilon<\min\{\bx^{\text{ideal}}_j(t)/2,-(\bb-A\bx^{\text{ideal}}(t))^{\top}A_j/\|A_j\|_2\}$ and define $\bx'=\bx^{\text{ideal}}(t) - \epsilon\cdot A_j/\|A_j\|_2$. Note that $\bx'$ lies in the original active cone. Observe that
		\begin{align*}
		\|\bb-A\bx'\|_2^2 &= \|\bb-A\bx^{\text{ideal}}(t) + \epsilon\cdot A_j/\|A_j\|_2\|_2^2\\
		&=\|\bb-A\bx^{\text{ideal}}(t)\|_2^2 + \|\epsilon\cdot A_j/\|A_j\|_2\|_2^2+2\epsilon\cdot(\bb-A\bx^{\text{ideal}}(t))^{\top}A_j/\|A_j\|_2\\
		&\leq\|\bb-A\bx^{\text{ideal}}(t)\|_2^2 + \epsilon^2 - 2\epsilon^2\\
		&<\|\bb-A\bx^{\text{ideal}}(t)\|_2^2
		\end{align*}
		which contradicts to the optimality of $\bx^{\text{ideal}}(t)$ since $\bx'$ is also a feasible solution. We conclude that $\bx^{\text{ideal}}_j(t)=0$. As a result, $A\bx^{\text{ideal}}(t)$ remains the same and $\|A\bx^{\text{ideal}}(t)\|_2$ won't decrease.
	\end{enumerate}
\end{proof}

The next lemma upper bounds the $\ell_2$ residual error of $\bx^\text{ideal}(t)$.
\begin{lemma}[convergence of residual error]\label{lemma:idealalgorithm-l2bound}
	There exists a polynomial $\alpha(\cdot)$ such that when $0<\alpha\leq \alpha(\frac{\gamma(A)}{n\cdot\lambda_{\max}})$, we have  for any $\epsilon>0$, when $t\geq\frac{m\cdot\OPT^{\ell_1}}{\epsilon\cdot\|\bb\|_2}$, $\|\bb-A\bx^{\text{ideal}}(t)\|_2\leq\epsilon\cdot\|\bb\|_2$.
\end{lemma}
\begin{proof}[Proof of Lemma~\ref{lemma:idealalgorithm-l2bound}]
	Assume the statement is wrong, \ie $\|\bb-A\bx^{\text{ideal}}(t)\|_2>\epsilon\cdot\|\bb\|_2$. Then by Lemma~\ref{lemma:idealalgorithm-primal-nondecreasing}, for any $0\leq s\leq t$,
	\begin{align*}
	\|\bb-A\bx^{\text{ideal}}(s)\|_2^2&=\|\bb\|_2^2-\|A\bx^{\text{ideal}}(s)\|_2^2\\
	&\geq\|\bb\|_2^2-\|A\bx^{\text{ideal}}(t)\|_2^2\\
	&=\|\bb-A\bx^{\text{ideal}}(t)\|_2^2>\epsilon^2\cdot\|\bb\|_2^2.
	\end{align*}
	Since $t\geq\frac{\OPT^{\ell_1}}{\epsilon\cdot\|\bb\|_2}$, by Lemma~\ref{lem:strict improvement},
	\begin{align*}
	\bb^\top\left(\bv^\text{ideal}(t)+\sum_{d\in[m-1]}\bv^\text{auxiliary}_d(t)\right) &= \int_0^t\bb^\top d\bv^\text{ideal}(t)+\sum_{d\in[m-1]}\int_0^t\bb^\top d\bv^\text{auxiliary}_d(t)\\
	&> t\cdot\epsilon\cdot\|\bb\|_2\geq m\cdot\OPT^{\ell_1},
	\end{align*}
	which is a contradiction to the optimality of $\OPT^{\ell_1}$ since $\bb^\top\bv^{\text{ideal}}(t),\bb^\top\bv^\text{auxiliary}_d(t)\leq\OPT^{\ell_1}$ for all $d\in[m-1]$. As a result, we conclude that $\|\bb-A\bx^{\text{ideal}}(t)\|_2\leq\epsilon\cdot\|\bb\|_2$.
\end{proof}

Finally, the following lemma shows that the $\ell_1$ error of $\bx^\text{ideal}(t)$ can be upper bounded by the $\ell_2$ error via the strong duality of $\ell_1$ minimization problem and perturbation trick.
\begin{lemma}[convergence of $\ell_1$ error]\label{lemma:idealalgorithm-OPTbounds}
	For any $t\geq0$,
	\begin{equation}
	\left|\|\bx^{\text{ideal}}(t)\|_1-\OPT^{\ell_1}\right|\leq\sqrt{\frac{n}{\lambda_{\min}}}\cdot\|\bb-A\bx^{\text{ideal}}(t)\|_2
	\end{equation}
\end{lemma}
\begin{proof}[Proof sketch]
	The proof of Lemma~\ref{lemma:idealalgorithm-OPTbounds} consists of two steps. First, we show that the primal and the dual solution pair of ideal SNN at time $t$ is the optimal solution pair of a \textit{perturbed $\ell_1$ minimization problem} defined as shifting the $\bb$ in the constraint $A\bx=\bb$ to $A\bx^{\text{ideal}}(t)$. See~\eqref{op:basispursuit-ideal-perturbed} for the definition of the perturbed program. Next, by the standard perturbation theorem from optimization, we can upper bound $\|\bx^{\text{ideal}}(t)\|_1$ with the distance between the original program and the perturbed program. Specifically, the difference induced by the perturbation is related to the $\ell_2$ norm of the differnce between $\bb$ and $A\bx^{\text{ideal}}(t)$, which is exactly the residual error. As a result, we know that the difference between the optimal value of the original $\ell_1$ minimization program and that of the perturbed program will converge to 0. Namely, we yield a convergence of $\|\bx^{\text{ideal}}(t)\|_1$ to $\OPT^{\ell_1}$. See Section~\ref{sec:missing proofs convergent analysis l1} for more details.
\end{proof}

Finally, we can prove the main theorem in this section as follows.
\begin{proof}[Proof of Theorem~\ref{thm:l1}]
	Pick $t_0=\Theta(\frac{m\cdot\sqrt{n}\cdot\|\bb\|_2}{\epsilon\cdot\sqrt{\lambda_{\min}\cdot\OPT^{\ell_1}}})$. By Lemma~\ref{lemma:idealalgorithm-l2bound}, for any $t\geq t_0$, we can upper bound the $\ell_2$ residual error by
	$$\|\bb-A\bx^\text{ideal}(t)\|_2\leq\sqrt{\frac{\lambda_{\min}}{n}}\cdot\frac{\epsilon}{10}\cdot\OPT^{\ell_1}.$$
	Next, by Lemma~\ref{lemma:idealalgorithm-OPTbounds}, we can then upper bound the $\ell_1$ error by
	$$\left|\|\bx^\text{ideal}(t)\|_1-\OPT^{\ell_1}\right|\leq\sqrt{\frac{n}{\lambda_{\min}}}\cdot\|\bb-A\bx^\text{ideal}(t)\|_2\leq\frac{\epsilon}{10}\cdot\OPT^{\ell_1}.$$
	Now, the only thing left is connecting the ideal solution $\bx^\text{ideal}(t)$ to the firing rate $\bx(t)$. First, divide $\bx(t)$ into two parts: the firing rate $\bx^{[0,t_0]}$ before time $t_0$ and the firing rate $\bx^{(t_0,t]}$ from time $t_0$ to $t$. That is, $\bx(t) = \frac{t_0}{t}\cdot\bx^{[0,t_0]}+\frac{t-t_0}{t}\cdot\bx^{(t_0,t]}$.
	
	Note that after $t_0\geq\Omega(\frac{m\cdot\sqrt{n}\cdot\|\bb\|_2}{\epsilon\cdot\sqrt{\lambda_{\min}\cdot\OPT^{\ell_1}}})$, the ideal solution has $\ell_1$ norm at most $(1+\epsilon)\cdot\OPT^{\ell_1}$. Thus, $\|\bx^{(t_0,t]}\|_1\leq(1+(1+O(\frac{1}{t}))\cdot\frac{\epsilon}{10})\cdot\OPT^{\ell_1}\leq(1+\frac{\epsilon}{5})\cdot\OPT^{\ell_1}$.
	As for $\bx^{[0,t_0]}$, from Lemma~\ref{lemma:idealalgorithm-OPTbounds}, we have $\|\bx^{[0,t_0]}\|_1\leq\OPT^{\ell_1}+\sqrt{\frac{n}{\lambda_{\min}}}\cdot\|\bb\|_2$.
	Combine the two, we have
	$$
	\left|\|\bx(t)\|_1-\OPT^{\ell_1}\right|\leq\frac{\epsilon\cdot\OPT^{\ell_1}}{5} + \frac{t_0\cdot\left(\OPT^{\ell_1}+\sqrt{\frac{n}{\lambda_{\min}}}\cdot\|\bb\|_2\right)}{t}\leq\epsilon\cdot\OPT^{\ell_1},
	$$
	where the last inequality holds since $t\geq\Omega(\frac{m^2\cdot n\cdot\|\bb\|_2^2}{\epsilon^2\cdot\lambda_{\min}\cdot\OPT^{\ell_1}})$. This completes the proof for Theorem~\ref{thm:l1}.
\end{proof}

\section{A simple SNN algorithm for the non-negative least squares}\label{sec:quadratic program}
In the introduction, we claim that we can show that the firing rate of one-sided SNN will converge to the solution of non-negative least squares problem. In this section, we are going to formally prove this Theorem~\ref{thm:quadratic program informal}. Recall that the non-negative least squares is defined as follows.
\begin{equation}\label{op:quadratic program proof}
\begin{aligned}
& \underset{\bx\in\R^n}{\text{minimize}}
& & \|A\bx-\bb\|_2^2 \\
& \text{subject to}
& & \bx_i\geq0,\ \forall i\in[n],
\end{aligned}
\end{equation}
We start with formally state the result into the following theorem.

\begin{theorem}\label{thm:quadratic program} 
	Given $A\in\R^{m\times n}$ and $\bb\in\R^m$ where all the row of $A$ has unit norm. Let $\gamma(A)\geq0$ be the niceness parameter of $A$ defined later in Definition~\ref{def:nice}. Suppose $\gamma(A)>0$. There exists a polynomial $\alpha(\cdot)$ such that for any $t\geq0$, let $\bx(t)$ be the firing rate of a simple continuous SNN with $C=\ATA$, $\bI=\ATb$, $\eta=1$, and $0<\alpha\leq \alpha(\frac{\gamma(A)}{n\cdot\lambda_{\max}})$. For any $\epsilon>0$, when $t\geq\frac{\sqrt{\lambda_{\max}\cdot n}}{\epsilon\cdot\lambda_{\min}\cdot\|\bb\|_2}$, then $\bx(t)$ is an $\epsilon$-approximation solution to the non-negative least squares problem.
	
	Here, we say $\bx$ is an $\epsilon$-approximation\footnote{The reason why we define in this way is to handle the case where the program has many solutions. In such case the only \textit{unique} thing is that $\|A\bx-\bb\|_2$ are all the same among these optimal solutions.} solution if for any optimal solution $\bx^*$ of the above program $\|A\bx-A\bx^*\|_2\leq\epsilon\cdot\|\bb\|_2$.
\end{theorem}

The proof follows from similar idea of ideal coupling. With the dual SNN view from Section~\ref{sec:dual SNN}, we can use Lemma~\ref{lem:ideal SNN unchaged} to control the behavior of dual SNN and thus have a good control on its dynamics..

\begin{proof}[Proof of Theorem~\ref{thm:quadratic program}]
	Given $A\in\R^{m\times n}$ and $\bb\in\R^m$, we first define the \textit{conic projection} of $\bb$ on the cone spanned by the column of $A$ as follows.
	\begin{equation*}
	\bx^+ = \argmin_{\bx\in\R^n_{\geq0}} \|A\bx-\bb\|_2^2.
	\end{equation*}
	Here, $\bx^+$ is the optimal solution of~\eqref{op:quadratic program proof} and we let $\bb^+=A\bx^+$ which is the conic projection of $\bb$ on the cone spanned by the column of $A$. Note that $\bx^+$ is also the optimal solution of the following optimization program with minimum value to be $0$.
	\begin{equation}\label{op:quadratic shift}
	\min_{\bx}\|A\bx-\bb^+\|_2^2.
	\end{equation}
	Given a simple SNN with $C=\ATA$ and $\bI=A^\top\bb$, the dual SNN as defined in Section~\ref{sec:dual SNN} would be
	\begin{equation*}
	\bv(t) = t\cdot\left(\bb-A\bx(t)\right).
	\end{equation*}
	Define $\bv^+(t)=\bv(t)-t\cdot\left(\bb-\bb^+\right)=t\cdot\left(\bb^+-A\bx(t)\right)$. It turns out that $\|\frac{\bv^+(t)}{t}\|_2^2=\|\bb^+-A\bx(t)\|_2^2$ is the residual error of $\bx(t)$ in solving~\eqref{op:quadratic shift}. That is, to prove Theorem~\ref{thm:quadratic program}, it suffices to show that $\|\bv^+(t)\|_2^2$ converges to $0$. We put this into the lemma below.
	\begin{lemma}\label{lem:quadratic dual bounded}
		With the conditions stated in Theorem~\ref{thm:quadratic program informal}, for any $t\geq0$, we have
		$$
		\bv^+(t)\in\left\{\bv\in\R^m:\ A_i^\top\bv\leq1,\ \forall i\in[n]\right\}\cap\left\{\bv=\sum_{i\in[n]}\alpha_iA_i:\ \alpha_i\geq0,\ \forall i\in[n] \right\}.
		$$
		Especially, we have $\|\bv^+(t)\|_2\leq\frac{\sqrt{\lambda_{\max}\cdot n}}{\lambda_{\min}}$.
	\end{lemma}
	
	From Lemma~\ref{lem:quadratic dual bounded}, we have $\|\bb^+-A\bx(t)\|_2\leq\frac{\sqrt{\lambda_{\max}\cdot n}}{\lambda_{\min}\cdot t}$. Let $\bx^*$ be the optimal solution of~\eqref{op:quadratic program proof} (which is also the optimal solution of~\eqref{op:quadratic shift} as we argued before), we have $\|\bb^+-A\bx^*\|_2=0$. By triangle inequality, we have $\|A\bx(t)-A\bx^*\|_2\leq\frac{\sqrt{\lambda_{\max}\cdot n}}{\lambda_{\min}\cdot t}$. When $t\geq\frac{\sqrt{\lambda_{\max}\cdot n}}{\epsilon\cdot\lambda_{\min}\cdot\|\bb\|_2}$, we have $\|A\bx-A\bx^*\|_2\leq\epsilon\cdot\|\bb\|_2$.
	This completes the proof of Theorem~\ref{thm:quadratic program}.
\end{proof}

\subsection{Proof of Lemma~\ref{lem:quadratic dual bounded}}
\begin{proof}[Proof of Lemma~\ref{lem:quadratic dual bounded}]
	The proof is based on induction on $t\geq0$. For the base case where $t=0$, the lemma is trivially true. Suppose the lemma holds for some $t\geq0$, consider $t+dt$. Note that
	$$
	\bv^+(t+dt) = \bv^+(t) - \alpha A\bs(t) + \bb^+dt.
	$$
	By Lemma~\ref{lem:ideal SNN unchaged}, we have
	$$
	\bv^+(t) - \alpha A\bs(t)\in\left\{\bv\in\R^m:\ A_i^\top\bv\leq1,\ \forall i\in[n]\right\}\cap\left\{\bv=\sum_{i\in[n]}\alpha_iA_i:\ \alpha_i\geq0,\ \forall i\in[n] \right\}
	$$
	from the induction hypothesis. As $\bb^+\in\left\{\bv=\sum_{i\in[n]}\alpha_iA_i:\ \alpha_i\geq0,\ \forall i\in[n] \right\}$ and $A_i^\top\left(\bv^+(t) - \alpha A\bs(t)\right)<1$ due to the spiking rule, we have
	$$
	\bv^+(t+dt)\in\left\{\bv\in\R^m:\ A_i^\top\bv\leq1,\ \forall i\in[n]\right\}\cap\left\{\bv=\sum_{i\in[n]}\alpha_iA_i:\ \alpha_i\geq0,\ \forall i\in[n] \right\}.
	$$
	Note that the largest $\ell_2$ norm in the above intersection is at most the largest $\ell_2$ norm in the dual polytope $\{\bv:\ \|A^\top\bv\|_\infty\leq1\}$. Thus, $\|\bv^+(t)\|_2\leq\frac{\sqrt{\lambda_{\max}\cdot n}}{\lambda_{\min}}$.
\end{proof}

\subparagraph*{Acknowledgements.}
The authors would like to thank Tsung-Han Lin, Zhenming Liu, Luca  Trevisan, Richard Peng, Yin-Hsun Huang, and Tao Xiao for useful discussions related to this paper. We are also thankful to the anonymous reviewer from ITCS 2019 for various useful comments and pointing out the inverse quasi-polynomial/exponential upper bound for the $\gamma$ of matrix sampled from RSM.

\bibliographystyle{alpha}
\bibliography{mybib}

\newcommand{\etalchar}[1]{$^{#1}$}
\begin{thebibliography}{FTHVVB03}

\bibitem[Abe91]{abeles1991corticonics}
Moshe Abeles.
\newblock {\em Corticonics: Neural circuits of the cerebral cortex}.
\newblock Cambridge University Press, 1991.

\bibitem[AS94]{allen1994evaluation}
Christina Allen and Charles~F Stevens.
\newblock An evaluation of causes for unreliability of synaptic transmission.
\newblock {\em Proceedings of the National Academy of Sciences},
  91(22):10380--10383, 1994.

\bibitem[Ban16]{banerjee2016learning}
Arunava Banerjee.
\newblock Learning precise spike train--to--spike train transformations in
  multilayer feedforward neuronal networks.
\newblock {\em Neural computation}, 28(5):826--848, 2016.

\bibitem[BBNM11]{buesing2011neural}
Lars Buesing, Johannes Bill, Bernhard Nessler, and Wolfgang Maass.
\newblock Neural dynamics as sampling: a model for stochastic computation in
  recurrent networks of spiking neurons.
\newblock {\em PLoS Comput Biol}, 7(11):e1002211, 2011.

\bibitem[BDM13]{barrett2013firing}
David~G Barrett, Sophie Den{\`e}ve, and Christian~K Machens.
\newblock Firing rate predictions in optimal balanced networks.
\newblock In {\em Advances in Neural Information Processing Systems}, pages
  1538--1546, 2013.

\bibitem[BIP15]{binas2015spiking}
Jonathan Binas, Giacomo Indiveri, and Michael Pfeiffer.
\newblock Spiking analog vlsi neuron assemblies as constraint satisfaction
  problem solvers.
\newblock {\em arXiv preprint arXiv:1511.00540}, 2015.

\bibitem[BL03]{brunel2003firing}
Nicolas Brunel and Peter~E Latham.
\newblock Firing rate of the noisy quadratic integrate-and-fire neuron.
\newblock {\em Neural Computation}, 15(10):2281--2306, 2003.

\bibitem[BMF{\etalchar{+}}17]{bengio2017stdp}
Yoshua Bengio, Thomas Mesnard, Asja Fischer, Saizheng Zhang, and Yuhuai Wu.
\newblock Stdp-compatible approximation of backpropagation in an energy-based
  model.
\newblock {\em Neural computation}, 29(3):555--577, 2017.

\bibitem[BMV12]{bonifaci2012physarum}
Vincenzo Bonifaci, Kurt Mehlhorn, and Girish Varma.
\newblock Physarum can compute shortest paths.
\newblock {\em Journal of Theoretical Biology}, 309:121--133, 2012.

\bibitem[BPLG16]{biswas2016simple}
Anmol Biswas, Sidharth Prasad, Sandip Lashkare, and Udayan Ganguly.
\newblock A simple and efficient snn and its performance \& robustness
  evaluation method to enable hardware implementation.
\newblock {\em arXiv preprint arXiv:1612.02233}, 2016.

\bibitem[BRVSW91]{bialek1991reading}
William Bialek, Fred Rieke, RR~De~Ruyter Van~Steveninck, and David Warland.
\newblock Reading a neural code.
\newblock {\em Science}, 252(5014):1854--1857, 1991.

\bibitem[BS98]{bonnans1998optimization}
J~Fr{\'e}d{\'e}ric Bonnans and Alexander Shapiro.
\newblock Optimization problems with perturbations: A guided tour.
\newblock {\em SIAM review}, 40(2):228--264, 1998.

\bibitem[BT09]{beck2009fast}
Amir Beck and Marc Teboulle.
\newblock A fast iterative shrinkage-thresholding algorithm for linear inverse
  problems.
\newblock {\em SIAM journal on imaging sciences}, 2(1):183--202, 2009.

\bibitem[BtN05]{booij2005gradient}
Olaf Booij and Hieu tat Nguyen.
\newblock A gradient descent rule for spiking neurons emitting multiple spikes.
\newblock {\em Information Processing Letters}, 95(6):552--558, 2005.

\bibitem[BV04]{boyd2004convex}
Stephen Boyd and Lieven Vandenberghe.
\newblock {\em Convex optimization}.
\newblock Cambridge university press, 2004.

\bibitem[CDS01]{chen2001atomic}
Scott~Shaobing Chen, David~L Donoho, and Michael~A Saunders.
\newblock Atomic decomposition by basis pursuit.
\newblock {\em SIAM review}, 43(1):129--159, 2001.

\bibitem[Cha09]{chazelle2009natural}
Bernard Chazelle.
\newblock Natural algorithms.
\newblock In {\em Proceedings of the twentieth Annual ACM-SIAM Symposium on
  Discrete Algorithms}, pages 422--431. Society for Industrial and Applied
  Mathematics, 2009.

\bibitem[Cha12]{chazelle2012natural}
Bernard Chazelle.
\newblock Natural algorithms and influence systems.
\newblock {\em Communications of the ACM}, 55(12):101--110, 2012.

\bibitem[DC15]{diehl2015unsupervised}
Peter~U Diehl and Matthew Cook.
\newblock Unsupervised learning of digit recognition using
  spike-timing-dependent plasticity.
\newblock {\em Frontiers in computational neuroscience}, 9:99, 2015.

\bibitem[Fit61]{fitzhugh1961impulses}
Richard FitzHugh.
\newblock Impulses and physiological states in theoretical models of nerve
  membrane.
\newblock {\em Biophysical journal}, 1(6):445--466, 1961.

\bibitem[FSW08]{faisal2008noise}
A~Aldo Faisal, Luc~PJ Selen, and Daniel~M Wolpert.
\newblock Noise in the nervous system.
\newblock {\em Nature reviews neuroscience}, 9(4):292--303, 2008.

\bibitem[FTHVVB03]{fourcaud2003spike}
Nicolas Fourcaud-Trocm{\'e}, David Hansel, Carl Van~Vreeswijk, and Nicolas
  Brunel.
\newblock How spike generation mechanisms determine the neuronal response to
  fluctuating inputs.
\newblock {\em Journal of Neuroscience}, 23(37):11628--11640, 2003.

\bibitem[Ger95]{gerstner1995time}
Wulfram Gerstner.
\newblock Time structure of the activity in neural network models.
\newblock {\em Physical review E}, 51(1):738, 1995.

\bibitem[GM08]{gollisch2008rapid}
Tim Gollisch and Markus Meister.
\newblock Rapid neural coding in the retina with relative spike latencies.
\newblock {\em science}, 319(5866):1108--1111, 2008.

\bibitem[Hei91]{heiligenberg1991neural}
Walter Heiligenberg.
\newblock {\em Neural nets in electric fish}.
\newblock MIT press Cambridge, MA, 1991.

\bibitem[HH52]{hodgkin1952quantitative}
Alan~L Hodgkin and Andrew~F Huxley.
\newblock A quantitative description of membrane current and its application to
  conduction and excitation in nerve.
\newblock {\em The Journal of physiology}, 117(4):500, 1952.

\bibitem[Hop95]{hopfield1995pattern}
John~J Hopfield.
\newblock Pattern recognition computation using action potential timing for
  stimulus representation.
\newblock {\em Nature}, 376(6535):33, 1995.

\bibitem[HR84]{hindmarsh1984model}
James~L Hindmarsh and RM~Rose.
\newblock A model of neuronal bursting using three coupled first order
  differential equations.
\newblock {\em Proc. R. Soc. Lond. B}, 221(1222):87--102, 1984.

\bibitem[I{\etalchar{+}}03]{izhikevich2003simple}
Eugene~M Izhikevich et~al.
\newblock Simple model of spiking neurons.
\newblock {\em IEEE Transactions on neural networks}, 14(6):1569--1572, 2003.

\bibitem[JHM14]{jonke2014theoretical}
Zeno Jonke, Stefan Habenschuss, and Wolfgang Maass.
\newblock A theoretical basis for efficient computations with noisy spiking
  neurons.
\newblock {\em arXiv preprint arXiv:1412.5862}, 2014.

\bibitem[JHM16]{jonke2016solving}
Zeno Jonke, Stefan Habenschuss, and Wolfgang Maass.
\newblock Solving constraint satisfaction problems with networks of spiking
  neurons.
\newblock {\em Frontiers in neuroscience}, 10, 2016.

\bibitem[KGH97]{kistler1997reduction}
Werner~M Kistler, Wulfram Gerstner, and J~Leo~van Hemmen.
\newblock Reduction of the hodgkin-huxley equations to a single-variable
  threshold model.
\newblock {\em Neural computation}, 9(5):1015--1045, 1997.

\bibitem[KGM16]{kheradpisheh2016bio}
Saeed~Reza Kheradpisheh, Mohammad Ganjtabesh, and Timoth{\'e}e Masquelier.
\newblock Bio-inspired unsupervised learning of visual features leads to robust
  invariant object recognition.
\newblock {\em Neurocomputing}, 205:382--392, 2016.

\bibitem[KS93]{kuwabara1993delay}
Nobuyuki Kuwabara and Nobuo Suga.
\newblock Delay lines and amplitude selectivity are created in subthalamic
  auditory nuclei: the brachium of the inferior colliculus of the mustached
  bat.
\newblock {\em Journal of neurophysiology}, 69(5):1713--1724, 1993.

\bibitem[Lap07]{lapicque1907recherches}
Louis Lapicque.
\newblock Recherches quantitatives sur l’excitation {\'e}lectrique des nerfs
  trait{\'e}e comme une polarisation.
\newblock {\em J. Physiol. Pathol. Gen}, 9(1):620--635, 1907.

\bibitem[LM18]{LM18}
Nancy Lynch and Cameron Musco.
\newblock A basic compositional model for spiking neural networks.
\newblock {\em arXiv preprint arXiv:1808.03884}, 2018.

\bibitem[LMP17a]{LMP17BDA}
Nancy Lynch, Cameron Musco, and Merav Parter.
\newblock Spiking neural networks: An algorithmic perspective.
\newblock In {\em Workshop on Biological Distributed Algorithms (BDA), July
  28th, 2017, Washington DC, USA}, 2017.

\bibitem[LMP17b]{LMP17ITCS}
Nancy~A. Lynch, Cameron Musco, and Merav Parter.
\newblock Computational tradeoffs in biological neural networks:
  Self-stabilizing winner-take-all networks.
\newblock In {\em 8th Innovations in Theoretical Computer Science Conference,
  {ITCS} 2017, January 9-11, 2017, Berkeley, CA, {USA}}, pages 15:1--15:44,
  2017.

\bibitem[LMP17c]{LMP17DISC}
Nancy~A. Lynch, Cameron Musco, and Merav Parter.
\newblock Neuro-ram unit with applications to similarity testing and
  compression in spiking neural networks.
\newblock In {\em 31st International Symposium on Distributed Computing, {DISC}
  2017, October 16-20, 2017, Vienna, Austria}, pages 33:1--33:16, 2017.

\bibitem[LP16]{livnat2016sex}
Adi Livnat and Christos Papadimitriou.
\newblock Sex as an algorithm: the theory of evolution under the lens of
  computation.
\newblock {\em Communications of the ACM}, 59(11):84--93, 2016.

\bibitem[LPR{\etalchar{+}}14]{livnat2014satisfiability}
Adi Livnat, Christos Papadimitriou, Aviad Rubinstein, Gregory Valiant, and
  Andrew Wan.
\newblock Satisfiability and evolution.
\newblock In {\em Foundations of Computer Science (FOCS), 2014 IEEE 55th Annual
  Symposium on}, pages 524--530. IEEE, 2014.

\bibitem[LT18]{lin2018dictionary}
Tsung-Han Lin and Ping Tak~Peter Tang.
\newblock Dictionary learning by dynamical neural networks.
\newblock {\em arXiv preprint arXiv:1805.08952}, 2018.

\bibitem[Maa96]{maass1996lower}
Wolfgang Maass.
\newblock Lower bounds for the computational power of networks of spiking
  neurons.
\newblock {\em Neural computation}, 8(1):1--40, 1996.

\bibitem[Maa97a]{maass1997fast}
Wolfgang Maass.
\newblock Fast sigmoidal networks via spiking neurons.
\newblock {\em Neural Computation}, 9(2):279--304, 1997.

\bibitem[Maa97b]{maass1997networks}
Wolfgang Maass.
\newblock Networks of spiking neurons: the third generation of neural network
  models.
\newblock {\em Neural networks}, 10(9):1659--1671, 1997.

\bibitem[Maa99]{maass19992}
Wolfgang Maass.
\newblock Computing with spiking neurons.
\newblock {\em Pulsed neural networks}, 85, 1999.

\bibitem[Maa15]{maass2015spike}
Wolfgang Maass.
\newblock To spike or not to spike: That is the question.
\newblock {\em Proceedings of the IEEE}, 103(12):2219--2224, 2015.

\bibitem[MB01]{maass2001pulsed}
Wolfgang Maass and Christopher~M Bishop.
\newblock {\em Pulsed neural networks}.
\newblock MIT press, 2001.

\bibitem[ML81]{morris1981voltage}
Catherine Morris and Harold Lecar.
\newblock Voltage oscillations in the barnacle giant muscle fiber.
\newblock {\em Biophysical journal}, 35(1):193--213, 1981.

\bibitem[MMI15]{mostafa2015event}
Hesham Mostafa, Lorenz~K M{\"u}ller, and Giacomo Indiveri.
\newblock An event-based architecture for solving constraint satisfaction
  problems.
\newblock {\em Nature communications}, 6, 2015.

\bibitem[NYT00]{nakagaki2000intelligence}
Toshiyuki Nakagaki, Hiroyasu Yamada, and {\'A}gota T{\'o}th.
\newblock Intelligence: Maze-solving by an amoeboid organism.
\newblock {\em Nature}, 407(6803):470--470, 2000.

\bibitem[OF96]{olshausen1996emergence}
Bruno~A Olshausen and David~J Field.
\newblock Emergence of simple-cell receptive field properties by learning a
  sparse code for natural images.
\newblock {\em Nature}, 381(6583):607, 1996.

\bibitem[PMB12]{paugam2012computing}
H{\'e}lene Paugam-Moisy and Sander Bohte.
\newblock Computing with spiking neuron networks.
\newblock In {\em Handbook of natural computing}, pages 335--376. Springer,
  2012.

\bibitem[RJBO08]{rozell2008sparse}
Christopher~J Rozell, Don~H Johnson, Richard~G Baraniuk, and Bruno~A Olshausen.
\newblock Sparse coding via thresholding and local competition in neural
  circuits.
\newblock {\em Neural computation}, 20(10):2526--2563, 2008.

\bibitem[RT01]{rullen2001rate}
Rufin~Van Rullen and Simon~J Thorpe.
\newblock Rate coding versus temporal order coding: what the retinal ganglion
  cells tell the visual cortex.
\newblock {\em Neural computation}, 13(6):1255--1283, 2001.

\bibitem[RW99]{rieke1999spikes}
Fred Rieke and David Warland.
\newblock {\em Spikes: exploring the neural code}.
\newblock MIT press, 1999.

\bibitem[SN94]{shadlen1994noise}
Michael~N Shadlen and William~T Newsome.
\newblock Noise, neural codes and cortical organization.
\newblock {\em Current opinion in neurobiology}, 4(4):569--579, 1994.

\bibitem[SRH13]{shapero2013configurable}
Samuel Shapero, Christopher Rozell, and Paul Hasler.
\newblock Configurable hardware integrate and fire neurons for sparse
  approximation.
\newblock {\em Neural Networks}, 45:134--143, 2013.

\bibitem[SS17]{shrestha2017robust}
Sumit~Bam Shrestha and Qing Song.
\newblock Robust learning in spikeprop.
\newblock {\em Neural Networks}, 86:54--68, 2017.

\bibitem[Ste65]{stein1965theoretical}
Richard~B Stein.
\newblock A theoretical analysis of neuronal variability.
\newblock {\em Biophysical Journal}, 5(2):173, 1965.

\bibitem[SZHR14]{shapero2014optimal}
Samuel Shapero, Mengchen Zhu, Jennifer Hasler, and Christopher Rozell.
\newblock Optimal sparse approximation with integrate and fire neurons.
\newblock {\em International journal of neural systems}, 24(05):1440001, 2014.

\bibitem[Tan16]{tang2016convergence}
Ping Tak~Peter Tang.
\newblock Convergence of lca flows to (c) lasso solutions.
\newblock {\em arXiv preprint arXiv:1603.01644}, 2016.

\bibitem[TDVR01]{thorpe2001spike}
Simon Thorpe, Arnaud Delorme, and Rufin Van~Rullen.
\newblock Spike-based strategies for rapid processing.
\newblock {\em Neural networks}, 14(6-7):715--725, 2001.

\bibitem[TFM96]{thorpe1996speed}
Simon Thorpe, Denis Fize, and Catherine Marlot.
\newblock Speed of processing in the human visual system.
\newblock {\em nature}, 381(6582):520, 1996.

\bibitem[TKN07]{tero2007mathematical}
Atsushi Tero, Ryo Kobayashi, and Toshiyuki Nakagaki.
\newblock A mathematical model for adaptive transport network in path finding
  by true slime mold.
\newblock {\em Journal of theoretical biology}, 244(4):553--564, 2007.

\bibitem[TLD17]{tang2017sparse}
Ping Tak~Peter Tang, Tsung-Han Lin, and Mike Davies.
\newblock Sparse coding by spiking neural networks: Convergence theory and
  computational results.
\newblock {\em arXiv preprint arXiv:1705.05475}, 2017.

\bibitem[TMS14]{teka2014neuronal}
Wondimu Teka, Toma~M Marinov, and Fidel Santamaria.
\newblock Neuronal spike timing adaptation described with a fractional leaky
  integrate-and-fire model.
\newblock {\em PLoS computational biology}, 10(3):e1003526, 2014.

\bibitem[ZMD11]{zylberberg2011sparse}
Joel Zylberberg, Jason~Timothy Murphy, and Michael~Robert DeWeese.
\newblock A sparse coding model with synaptically local plasticity and spiking
  neurons can account for the diverse shapes of v1 simple cell receptive
  fields.
\newblock {\em PLoS computational biology}, 7(10):e1002250, 2011.

\end{thebibliography}

\appendix
\section{Missing proofs for Theorem~\ref{thm:l1}}\label{sec:missing proofs}
\subsection{Proofs for the properties of ideal and auxiliary SNN}\label{sec:missing proofs ideal auxiliary SNN}
\begin{proof}[Proof of Lemma~\ref{lem:ideal partition}]
	Let us start with an observation on Definition~\ref{def:ideal partition} about the points on the boundary of  the ideal polytope $\mathcal{P}_{A,1-\tau}$.
	\begin{claim}\label{claim:ideal active PD}
		If $A$ is non-degenerate, then for any $\bv^\text{ideal}\in\partial\mathcal{P}_{A,1-\tau}$, $\text{rank}(A_{\Gamma(\bv^\text{ideal})})=|\Gamma(\bv^\text{ideal})|$. Thus, $A_{\Gamma(\bv^\text{ideal})}^\top A_{\Gamma(\bv^\text{ideal})}$ is positive definite.
	\end{claim}
	
	Next, let us show that for $\bv^\text{ideal}_1\neq\bv^\text{ideal}_2\in\mathcal{P}_{A,1-\tau}$, $S_{\bv^\text{ideal}_1}\cap S_{\bv^\text{ideal}_2}=\emptyset$. It is trivially true when at least one of them does not lie on the boundary\footnote{Note that $\bv^\text{ideal}$ does not lie on the boundary of $\mathcal{P}_{A,1-\tau}$ if and only if $\Gamma(\bv^\text{ideal})=\emptyset$.} of $\mathcal{P}_{A,1-\tau}$. Now, consider the case where both of them lie on the boundary of $\mathcal{P}_{A,1-\tau}$ and denote their active set as $\Gamma_1=\Gamma(\bv^\textit{ideal}_1)$ and $\Gamma_2=\Gamma(\bv^\textit{ideal}_2)$. To prove from contradiction, suppose there exists $\bv\in S_{\bv^\text{ideal}_1}\cap S_{\bv^\text{ideal}_2}$. By definition, we have
	\begin{align*}
	\bv &= \bv^\text{ideal}_1 + A_{\Gamma_1}^\top\bz_1\\
	&=\bv^\text{ideal}_2 + A_{\Gamma_2}^\top\bz_2,
	\end{align*}
	where $\bz_1,\bz_2\geq0$. Let $\Gamma=\Gamma_1\cap\Gamma_2$. Consider the following cases.
	\begin{itemize}
		\item ($\Gamma=\Gamma_1=\Gamma_2$) By Definition~\ref{def:ideal partition}, we have $A_\Gamma^\top\bv^\text{ideal}_1=A_\Gamma^\top\bv^\text{ideal}_2=\mathbf{1}$ and thus
		\begin{equation*}
		(\bz_2-\bz_1)^\top A_\Gamma^\top A_\Gamma(\bz_2-\bz_1)=(\bz_2-\bz_1)^\top A_\Gamma^\top(\bv^\text{ideal}_1-\bv^\text{ideal}_2)=0.
		\end{equation*}
		As $A_\Gamma^\top A_\Gamma$ is positive definite by Claim~\ref{claim:ideal active PD}, we have $\bz_1=\bz_2$ and $\bv^\text{ideal}_1=\bv^\text{ideal}_2$, which is a contradiction.
		
		\item ($\Gamma_1\neq\Gamma_2$) Without loss of generality, assume $\Gamma_1\backslash\Gamma\neq\emptyset$ and $\bz_1\neq\mathbf{0}$. By Definition~\ref{def:ideal partition}, we have
		\begin{align*}
		A_{\Gamma_1\backslash\Gamma_2}^\top\left(\bv^\text{ideal}_1-\bv^\text{ideal}_2\right)&>\mathbf{0},\\
		A_{\Gamma_2\backslash\Gamma_1}^\top\left(\bv^\text{ideal}_1-\bv^\text{ideal}_2\right)&\leq\mathbf{0},\\
		A_{\Gamma}^\top\left(\bv^\text{ideal}_1-\bv^\text{ideal}_2\right)&=\mathbf{0}.
		\end{align*}
		As $\bz_1\neq\mathbf{0}$, we then have
		\begin{align*}
		\|A_{\Gamma_2}\bz_2-A_{\Gamma_1}\bz_1\|_2^2 &= \left(A_{\Gamma_2}\bz_2-A_{\Gamma_1}\bz_1\right)^\top\left(\bv^\text{ideal}_1-\bv^\text{ideal}_2\right)\\
		&=\left(-A_{\Gamma_1\backslash\Gamma}\bz_1|_{\Gamma_1\backslash\Gamma}\right)^\top\left(\bv^\text{ideal}_1-\bv^\text{ideal}_2\right)\\
		&+\left(A_{\Gamma_2\backslash\Gamma}\bz_2|_{\Gamma_2\backslash\Gamma}\right)^\top\left(\bv^\text{ideal}_1-\bv^\text{ideal}_2\right)\\
		&+\left(A_{\Gamma}\bz_2|_{\Gamma}-A_{\Gamma}\bz_1|_{\Gamma}\right)^\top\left(\bv^\text{ideal}_1-\bv^\text{ideal}_2\right)\\
		&<0.
		\end{align*}
		Note that the reason why the last inequality holds is because $\left(-A_{\Gamma_1\backslash\Gamma}\bz_1|_{\Gamma_1\backslash\Gamma}\right)^\top\left(\bv^\text{ideal}_1-\bv^\text{ideal}_2\right)<0$.
	\end{itemize}
	Finally, it is easy to see that $\{S_{\bv^\text{ideal}}\}_{\bv^\text{ideal}\in\mathcal{P}_{A,1-\tau}}$ covers $\mathcal{P}_{A,1}$ and thus we conclude that it is indeed a partition for $\mathcal{P}_{A,1}$.
\end{proof}

\subsection{Proofs for the convergent analysis of solving $\ell_1$ minimization}\label{sec:missing proofs convergent analysis l1}

\begin{proof}{Lemma~\ref{lemma:idealalgorithm-OPTbounds}}
	For any $t\geq0$, define the following perturbed program of (\ref{op:basispursuit}) and its dual.
	\vspace{3mm}
	\begin{minipage}{\linewidth}
		\begin{minipage}{0.45\linewidth}
			\begin{equation}\label{op:basispursuit-ideal-perturbed}
			\begin{aligned}
			& \underset{\bx}{\text{minimize}}
			& & \|\bx\|_1 \\
			& \text{subject to}
			& & A\bx-A\bx^{\text{ideal}}(t)=0
			\end{aligned}
			\end{equation}
		\end{minipage}
		\begin{minipage}{0.45\linewidth}
			\begin{equation}\label{op:basispursuit-ideal-perturbed-dual}
			\begin{aligned}
			& \underset{\bv\in\R^m}{\text{maximize}}
			& & (A\bx^{\text{ideal}}(t))\top\bv \\
			& \text{subject to}
			& & \|A^\top\bv\|_{\infty}\leq1.
			\end{aligned}
			\end{equation}
		\end{minipage}
	\end{minipage}
	
	Note that $\bx^{\text{ideal}}(t)$ is treated as a given constant to the optimization program. It turns out that the ideal algorithm optimizes this primal-dual perturbed program at time $t$ with the following parameters.
	\begin{lemma}\label{lemma:proofofidealalgorithm-KKT}
		For any $t\geq0$, $(\bx^*,\bv^*) = (\bx^{\text{ideal}}(t),\bv^{\text{ideal}}(t))$ is the optimal solutions of (\ref{op:basispursuit-ideal-perturbed}).
	\end{lemma}
	\begin{proof}{Proof of Lemma~\ref{lemma:proofofidealalgorithm-KKT}}
		We simply check the KKT condition. Since the program can be rewritten as a linear program, it satisfies the regularity condition of the KKT condition.
		
		First, the primal and the dual feasibility can be verified by the dynamics of ideal algorithm. That is, $A\bx^*-A\bx^{\text{ideal}}(t)=0$ and $\|A^\top\bv^*\|_{\infty}\leq1$. Next, consider the Lagrangian of~\eqref{op:basispursuit-ideal-perturbed} as follows.
		\begin{align*}
		\mathcal{L}(\bx,\bv) &= \|\bx\|_1 - \bv^\top (A\bx-A\bx^{\text{ideal}}(t)),\\
		\nabla_{\bx}\mathcal{L}(\bx,\bv) &= \nabla\|\bx\|_1-A^\top \bv.
		\end{align*}
		Now, let's verify that the gradient of the Lagrangian over $\bx$ is vanishing at $(\bx^*,\bv^*) = (\bx^{\text{ideal}}(t),\bv^{\text{ideal}}(t))$. That is,  $0\in\nabla_{\bx}\mathcal{L}(\bx^*,\bv^*)=\nabla\|\bx\|_1-A^\top \bv$. Consider two cases as follows. For any $i\in[n]$,
		\begin{enumerate}[label=(\arabic*)]
			\item When $i,-i\notin\Gamma^{\text{ideal}}(t)$. We have $\Big(\bx^{\text{ideal}}(t) \Big)_i=0$, \ie the sub-gradient of the $i$th coordinate of $\|\bx^{\text{ideal}}(t)\|_1$ lies in $[-1,1]$. As $A_i\top\bv^{\text{ideal}}(t)\in[-1,1]$, we have $A_i\top\bv^{\text{ideal}}(t)\in\partial_{\bx_i}\|\bx^{\text{ideal}}(t)\|_1$.
			\item When $i\in\Gamma^{\text{ideal}}(t)$ (or $-i\in\Gamma^{\text{ideal}}(t)$). We have $A_i\top\bv^{\text{ideal}}(t)=1$ (or $A_i\top\bv^{\text{ideal}}(t)=-1$). As $\text{sgn}\Big(\bx^{\text{ideal}}(t) \Big)_i=1$ (or $\text{sgn}\Big(\bx^{\text{ideal}}(t) \Big)_i=-1$), we have $A_i\top\bv^{\text{ideal}}(t)=\text{sgn}\Big(\bx^{\text{ideal}}(t) \Big)_i=\partial_{\bx_i}\|\bx^{\text{ideal}}(t)\|_1$.
		\end{enumerate}
		Finally, the complementary slackness is satisfied because $A\bx^*-A\bx^{\text{ideal}}(t)=0$.
		As a result, we conclude that $(\bx^*,\bv^*)$ is the optimal solution of (\ref{op:basispursuit-ideal-perturbed}).
	\end{proof}
	
	Next, we are going to use the perturbation lemma in the Chapter 5.6 of~\cite{boyd2004convex} stated as follows.
	
	\begin{lemma}[perturbation lemma]
		Given the following two optimization programs.\\
		\begin{minipage}{\linewidth}
			\begin{minipage}{0.45\linewidth}
				\begin{equation}\label{op:perturbation-original}
				\begin{aligned}
				& \underset{\bx}{\text{minimize}}
				& & f(\bx) \\
				& \text{subject to}
				& & h(\bx)=\mathbf{0}.
				\end{aligned}
				\end{equation}
			\end{minipage}
			\begin{minipage}{0.45\linewidth}
				\begin{equation}\label{op:perturbation-perturbed}
				\begin{aligned}
				& \underset{\bx}{\text{minimize}}
				& & f(\bx) \\
				& \text{subject to}
				& & h(\bx)=\by.
				\end{aligned}
				\end{equation}
			\end{minipage}
		\end{minipage}
		Let $\OPT^{\text{original}}$ be the optimal value of the original program~\eqref{op:perturbation-original} and $\OPT^{\text{perturbed}}$ be the optimal value of the perturbed program~\eqref{op:perturbation-perturbed}. Let $\bv^*$ be the optimal dual value of the perturbed program~\eqref{op:perturbation-perturbed}. We have
		\begin{equation}
		\OPT^{\text{original}}\geq\OPT^{\text{perturbed}}+\by^\top\bv^*.
		\end{equation}
	\end{lemma}
	
	Now, think of~\eqref{op:basispursuit} as the original program and~\eqref{op:basispursuit-ideal-perturbed} as the perturbed program. Namely, $f(\bx)=\|\bx\|_1$, $h(\bx)=A\bx-\bb$, and $\by=A\bx^{\text{ideal}}(t)-\bb$. By the perturbation lemma, we have
	\begin{align*}
	\OPT^{\ell_1}&\geq \|\bx^{\text{ideal}}(t)\|_1 +  \Big(A\bx^{\text{ideal}}(t)-\bb\Big)^\top\bv^{\text{ideal}}(t).
	\end{align*}
	As a result, the following upper bound holds.
	\begin{equation}\label{eq:proofofidealalgorithm-1}
	\|\bx^{\text{ideal}}(t)\|_1\leq\OPT^{\ell_1} + \|\bv^{\text{ideal}}(t)\|_2\cdot\|\bb-A\bx^{\text{ideal}}(t)\|_2.
	\end{equation}
	Finally, as $\bv^{\text{ideal}}(t)$ lies in the feasible region $\{\bv:\ A^\top \bv\|_{\infty}\leq1 \}$ and the range space of $A$, we can upper bound the $\|\bv^{\text{ideal}}(t)\|_2$ term in~\eqref{eq:proofofidealalgorithm-1} as follows.
	\begin{lemma}\label{lemma:proofofidealalgorithm-vbound}
		For any $\bv$ in the range space of $A$ and $\|A^\top \bv\|_{\infty}\leq1$, $\|\bv\|_2\leq\sqrt{\frac{n}{\lambda_{\min}}}$.
	\end{lemma}
	\begin{proof}{Proof of Lemma~\ref{lemma:proofofidealalgorithm-vbound}}
		As $\bv$ lies in the range space of $A$, we have $\|A^\top\bv\|_2\geq\sqrt{\lambda_{\min}}\|\bv\|_2$. Also, because $\|A^\top\bv\|_{\infty}\leq1$, we have $\|A^\top\bv\|_2\leq\sqrt{n}$. As a result, 
		\begin{equation*}
		\|\bv\|_2\leq\frac{\|A^\top\bv\|_2}{\sqrt{\lambda_{\min}}}\leq\sqrt{\frac{n}{\lambda_{\min}}}.
		\end{equation*}
	\end{proof}
	By~\eqref{eq:proofofidealalgorithm-1} and Lemma~\ref{lemma:proofofidealalgorithm-vbound}, Lemma~\ref{lemma:idealalgorithm-OPTbounds} holds.		
\end{proof}

\section{An inverse quasi-polynomial upper bound for the $\gamma$ of RSM}\label{sec:quasi poly ub for gamma of RSM}

In Lemma~\ref{lem:gamma lb of RSM}, we saw that $\gamma(A)>0$ with high probability when $A$ is sampled from the rotational symmetry model (RSM). As the choice of parameters (\textit{e.g.,} the spiking strength $\alpha$ and the discretization size $\Dt$) in Theorem~\ref{thm:l1} has a polynomial dependency on $\gamma(A)$, it would be nice if $\gamma(A)$ is as large as possible. However, in this section, we are going to show that for $A$ sampled from RSM, $\gamma(A)$ is upper bounded by inverse quasi-polynomial in $m$ if $n\geq\polylog(m)\cdot m$ and is upper bounded by inverse exponential in $m$ if $n\geq m^{1+\Omega(1)}$. We thank the anonymous reviewer from ITCS 2019 for pointing out the analysis of these upper bounds.

\begin{lemma}\label{lem:gamma ub of RSM}
	For any $m\in\N$ large enough, $0<\tau\leq\frac{m}{4}$, and $n\geq(2\log m/e^{-\tau})\cdot m$. Let $A\in\R^{m\times n}$ be a random matrix samples from RSM. Then, $\gamma(A)\leq e^{-\Omega(\tau\cdot\log m)}$ with high probability.
\end{lemma}
\begin{proof}[Proof of Lemma~\ref{lem:gamma ub of RSM}]
	The high-level idea of the analysis is iteratively looking at the correlation between the first column of $A$ and the other columns. In particular, divide the rest of columns of $A$ into $m$ buckets each of size $k=\floor*{\frac{n}{m}}$. This will give us $m$ buckets of $k$ independent unit vectors in $\R^m$. The idea is then projecting $A_1$ to the subspace spanned by each bucket one by one and argue that the length of the projection decreases by a non-trivial factor. Before doing the formal analysis, let us first prove the following claim about the distribution of the inner product of two random unit vectors in $\R^m$.
	
	\begin{claim}\label{claim:inner product of random unit vector}
		Let $\bv_1,\bv_2$ be two independent random unit vector in $\R^m$. For any $m$ large enough and $z\in[0,\frac{1}{4}]$, $\Pr[\langle\bv_1,\bv_2\rangle^2\leq z]\leq1-e^{-mz}$. 
	\end{claim}
	\begin{proof}[Proof of Claim~\ref{claim:inner product of random unit vector}]
		Let $Z=\langle\bv_1,\bv_2\rangle^2$, the probability of $Z= z$ for any $z\in[0,1]$ can be computed by the equation for the surface area on an unit ball in $\R^m$. Concretely, $\Pr[Z= z]$ is proportional to $(1-z)^{\frac{m-3}{2}}$. Thus, the probability of $Z\leq z$ is
		$$
		\Pr[Z\leq z] = \frac{\int_0^z (1-t)^{\frac{m-3}{2}}dt}{\int_0^1(1-t)^{\frac{m-3}{2}}dt} = 1-(1-z)^{\frac{m-1}{2}}.
		$$
		
		When $z\in[0,\frac{1}{4}]$, we can use the approximation $e^{-2z}\leq(1-z)\leq e^{-z}$ and get $\Pr[Z\leq z]\leq1-e^{-mz}$.
	\end{proof}
	
	Now, let us start with the first bucket of $k$ independent random unit vectors in $\R^m$. From Claim~\ref{claim:inner product of random unit vector}, we know that the probability of existing $\bv_1$ in the bucket such that $\langle A_1,\bv_1\rangle^2>\frac{\tau}{m}$ is at least $1-(1-e^{-\tau})^k$. Let $\Gamma_1=\{\bv_1\}$, with probability at least $1-(1-e^{-\tau})^k$, 
	$$\|A_1-\Pi_{\Gamma_1}A_1\|_2^2\leq(1-\langle A_1,\bv)\rangle^2)\cdot\|A_1\|_2^2\leq(1-\frac{\tau}{m})\cdot\|A_1\|_2^2.$$
	For the second bucket, we consider the subspace of $\R^m$ orthogonal to $\Gamma_1$. Using the same argument, we can find $\bv_2$ in the second bucket such that with probability at least $1-(1-e^{-m\cdot\frac{\tau}{m-1}})^k$, $\frac{\langle \Pi_{\Gamma_1}A_1,\bv_2\rangle}{\|\Pi_{\Gamma_1}A_1\|_2\cdot\|\bv_2\|_2}>\frac{\tau}{m-1}$ and thus 
	$$\|A_1-\Pi_{\Gamma_2}A_1\|_2^2\leq(1-\frac{\tau}{m-1})\cdot\|\Pi_{\Gamma_1}A_1\|_2^2\leq(1-\frac{\tau}{m-1})\cdot(1-\frac{\tau}{m})\cdot\|A_1\|_2^2.$$
	Repeat the above argument for $m-1$ times and apply union bound, we have
	$$
	\|A_1-\Pi_{\Gamma_s}A_1\|_2^2\leq\prod_{i=1}^s(1-\frac{\tau}{m-i+1})\cdot\|A_1\|_2^2\leq e^{-\Omega(\tau\cdot\log m)}\cdot\|A_1\|_2^2
	$$
	with probability at least
	$$1-\sum_{i=1}^{m-1}(1-e^{-m\cdot\frac{\tau}{m-i+1}})^k\geq1-(m-1)\cdot e^{-ke^{-\tau}}\geq1-e^{-ke^{-\tau}+\log m}.$$
	By our choice of $\tau$ and $n$, we have $\|A_1-\Pi_{\Gamma_s}A_i\|_2\leq e^{-\Omega(\tau\cdot\log m)}\cdot\|A_1\|_2$ with probability $1-o(1)$.
\end{proof}

\end{document}